\documentclass[nohyperref]{article}

\usepackage{microtype}
\usepackage{graphicx}
\usepackage{subfigure}
\usepackage{booktabs} 

\usepackage{hyperref}

\usepackage[accepted]{icml2022}

\usepackage{amsmath}
\usepackage{amssymb}
\usepackage{mathtools}
\usepackage{amsthm}

\usepackage[capitalize,noabbrev]{cleveref}
\usepackage{natbib}

\usepackage{booktabs}
\usepackage{graphicx}
\usepackage{multirow}
\usepackage{wrapfig,threeparttable}
\usepackage{capt-of}

\theoremstyle{plain}
\newtheorem{theorem}{Theorem}[section]

\newtheorem{lemma}[theorem]{Lemma}

\theoremstyle{definition}

\theoremstyle{remark}

\usepackage[textsize=tiny]{todonotes}

\icmltitlerunning{Permutation Search of Tensor Network Structures via Local Sampling}

\begin{document}

\twocolumn[
\icmltitle{Permutation Search of Tensor Network Structures\\ via Local Sampling}



\icmlsetsymbol{equal}{*}

\begin{icmlauthorlist}
\icmlauthor{Chao~Li}{equal,riken}
\icmlauthor{Junhua~Zeng}{equal,gut,riken}
\icmlauthor{Zerui~Tao}{tuat,riken}
\icmlauthor{Qibin~Zhao}{riken}
\end{icmlauthorlist}

\icmlaffiliation{riken}{RIKEN Center for Advanced Intelligence Project (RIKEN-AIP), Tokyo, Japan}
\icmlaffiliation{gut}{School of Automation, Guangdong University of Technology, Guangzhou, China}
\icmlaffiliation{tuat}{Tokyo University of Agriculture and Technology, Tokyo, Japan}

\icmlcorrespondingauthor{Qibin~Zhao}{qibin.zhao@riken.jp}
\icmlcorrespondingauthor{Chao~Li}{chao.li@riken.jp}

\icmlkeywords{Tensor Network, Tensor Decomposition}

\vskip 0.3in
]



\printAffiliationsAndNotice{\icmlEqualContribution} 

\begin{abstract}
	Recent works put much effort into \emph{tensor network structure search} (TN-SS), aiming to select suitable tensor network (TN) structures, involving the TN-ranks, formats, and so on, for the decomposition or learning tasks.
	In this paper, we consider a practical variant of TN-SS, dubbed \emph{TN permutation search}~(TN-PS), 
	in which we search for good mappings from tensor modes onto TN vertices (core tensors) for compact TN representations.
	We conduct a theoretical investigation of TN-PS and propose a practically-efficient algorithm to resolve the problem.
	Theoretically, we prove the counting and metric properties of search spaces of TN-PS, analyzing for the first time the impact of TN structures on these unique properties. 
	Numerically, we propose a novel \emph{meta-heuristic} algorithm, in which the searching is done by randomly sampling in a neighborhood established in our theory, and then recurrently updating the neighborhood  until convergence.
	Numerical results demonstrate that the new algorithm can reduce the required model size of TNs in extensive benchmarks, implying the improvement in the expressive power of TNs.
	Furthermore, the computational cost for the new algorithm is significantly less than that in~\cite{li2020evolutionary}.	
\end{abstract}

\section{Introduction}\label{sec:Intro}

Over the years, \emph{tensor network}~(TN) has been widely applied to various technical fields.
It enables us to resolve extremely high-dimensional problems, such as deep learning~\citep{novikov2015tensorizing,kossaifi2020tensor}, probability density estimation~\citep{glasser2019expressive,miller2021tensor,novikov2021tensor}, partial differential equations~\citep{bachmayr2016tensor,pmlr-v139-richter21a}, and quantum circuit simulation~\citep{markov2008simulating,huggins2019towards}, with acceptable computational and storage costs.
However, as an inevitable side-effect, practitioners have to face a notorious challenge when applying TNs to practical tasks: how to efficiently select the optimal TN structures from a massive quantity of candidates?

Recent works thus put effort into this challenge, in the heading of \emph{TN structure search}~(TN-SS).
Most recently, there have been studies, which focus on searching TN ranks~\cite{hashemizadeh2020adaptive,kodryan2020mars}, formats~\cite{hayashi2019exploring,li2020evolutionary}, and orders~\cite{li2020high,qiu2021memory} to achieve more compact representations.
These results also confirm numerically that the structures impact the expressive power~\citep{cohen2016expressive} of TNs in learning tasks.

In this paper, we consider a practical variant of \mbox{TN-SS}.
The goal is to improve the compactness and expressiveness of a TN while \emph{preserving} its format, such as tensor train~(TT, \citealt{oseledets2011tensor}) or tensor ring~(TR, \citealt{zhao2016tensor}).
Several existing works~\cite{zhao2016tensor,zheng2021fully} have noticed that the mapping from tensor modes onto the TN vertices, also known as core tensors, also influences the expressive power of the model.
To see this, we implement a toy experiment, in which TR is utilized to approximate a tensor of order four.
Figure~\ref{fig:cover} shows the required ranks for achieving the same approximation accuracy in the experiment.
We see that a good ``mode-vertex'' mapping (corresponding to \emph{Model 1}) would produce smaller ranks than the other two models.
It implies lower computational and storage costs and more promising generalization capability in learning tasks~\cite{khavari2021lower}.
This fact thus motivates this work for searching both the optimal TN-ranks and ``mode-vertex'' mappings.

Despite the potential benefit, searching for the optimal ``mode-vertex'' mappings is non-trivial in general.
For instance, there would be $\mathcal{O}(N!)$ different candidates for TR of order $N\geq{}3$ even though the optimal ranks are known.
It is apparently unacceptable to solve it by exhaustive search, particularly when combinatorially searching the TN-ranks is required as well.
Unlike TN-SS, it also appears new theoretical questions for the variant: how fast the scale of the mappings grows with the structures parameters, \textit{e.g.}, format, order and ranks?
Why does the growth rate change?
And what bounds the growth?

To this end, we conduct a thorough investigation of this variant, named \emph{TN permutation search}~(TN-PS), from both the theoretical and numerical aspects.
Theoretically, we answer the preceding questions by analyzing the counting property of the search space of TN-PS, proving a universal non-asymptotic bound for TNs in \emph{arbitrary} formats.
The result is helpful for fast estimation of the computational budget in searching.
We also establish the basic geometry for TN-PS with group-theoretical instruments, involving the (semi-)metric and neighborhood of the search space, such that the local searching can be applied to the task.

Numerically, we develop an efficient algorithm for TN-PS.
In contrast to the existing sampling-based methods~\cite{hayashi2019exploring,li2020evolutionary},
we draw samples in the established neighborhood to explore the `steepest-descent'' path of the landscape, thereby accelerating the searching procedure and decreasing the computational cost.
Experimental results on extensive benchmarks demonstrate that the proposed algorithm is \emph{unique} to resolving TN-PS consistently so far, and the required model evaluations are much fewer than the previous algorithm.
We summarize the main contributions of this work as follows:
\begin{itemize}
	\item We propose for the first time the problem of \emph{tensor-network permutation search} (TN-PS), an important variant of TN-SS in practice;
	\item We rigorously prove new theoretical properties for TN-PS, involving the counting, (semi-)metric, and neighborhood, revealing how TN structures impact these properties;
	\item We develop a local-sampling-based meta-heuristic, which significantly reduces the computational cost compared to  \cite{li2020evolutionary}.
\end{itemize}

\begin{figure}
	\centering
	\includegraphics[width=0.9\columnwidth]{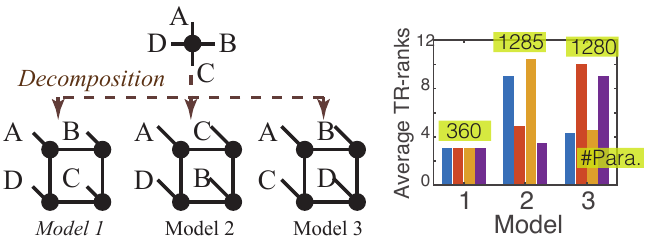}
	\caption{Impact of the ``mode-vertex'' relations on ranks in TR format.
		The chart shows the required ranks and number of parameters ($\#$Para.) in average. See supplementary materials for details.
	}\vskip -0.1in
	\label{fig:cover}
\end{figure}

\subsection{Related Works}
\noindent\textbf{Searching tensor-network (TN) structures.}
Searching the optimal structures for TNs is typically thought of as an extension of the rank selection problem for tensor learning~\citep{zhao2015bayesian,yokota2016smooth,zhao2016tensor,cheng2020novel,mickelin2020algorithms,cai2021blind,hawkins2021bayesian,li2021heuristic,long2021bayesian,sedighin2021adaptive}, which is widely known to be challenging, especially when the TN formats contain cycles~\citep{landsberg2011geometry,batselier2018trouble,Ye2019Tensor}.
More recently, several studies put much effort into this problem, \textit{i.e.}, TN-SS, for exploring unknown formats~\citep{hayashi2019exploring,hashemizadeh2020adaptive,kodryan2020mars,li2020evolutionary,nie2021adaptive}.
The latest works~\cite{razin2021implicit,razin2022implicit} also study the implicit regularization over TN-ranks.
Another line of work closely related to ours is those that study the partition issue for hierarchical Tucker~(HT, \citealt{falco2020geometry,haberstich2021active}) decomposition, which aims to search for the optimal tree structures.
Compared to these works, we focus on the search over the ``mode-vertex'' mappings, which have remained unexplored until now.

\noindent\textbf{Sampling-based optimization.}
Our new algorithm is inspired by zeroth-order optimization, which 
is also known as gradient-free optimization or bandit optimization. 
The methods can date back to stochastic hill-climbing~\cite{russell1995artificial}, followed by numerous evolutionary programming algorithms~\cite{back1996evolutionary}, and are restudied recently by~\citet{golovin2019gradientless} and applied to various machine learning tasks~\cite{liu2020primer,liu2020min,savarese2021domain,singh2021continuum}.
Inspired by the work~\cite{golovin2019gradientless}, we refine the sampling strategy for TN-PS by taking the unique property of the neighborhood into account but maintaining its original simplicity and efficiency.

\section{Preliminaries}
We first summarize notations and elementary results used throughout the paper.
After that, definitions related to \emph{tensor networks} (TNs) are reviewed for the self-contained purpose.

Throughout the paper, we use blackboard letters, such as $\mathbb{G}$ and $\mathbb{S}$, to denote sets of subjects.
With additional structures, they are also used to represent specific algebraic subjects according to the context, such as groups, fields or linear spaces.
In particular, we use $\mathbb{S}_N$, $\mathbb{R}$, $\mathbb{Z}^+$, and $\mathbb{R}^{I_1\times{}I_2\times\cdots\times{}I_N}$ to represent the \emph{symmetric group} of order $N$, the real field, positive integers and the real linear space of dimension $I_1\times{}I_2\times\cdots\times{}I_N$, respectively.
The \emph{size} of a finite set $\mathbb{A}$ is denoted by $\vert\mathbb{A}\vert$, and the \emph{Cartesian product} of two sets $\mathbb{A}$ and $\mathbb{B}$ is denoted by $\mathbb{A}\times{}\mathbb{B}$.
We say two sets, \textit{e.g.}, $\mathbb{A,B}$, are equivalent if there exists a bijective mapping from $\mathbb{A}$ and $\mathbb{B}$, and sometimes write $\mathbb{A}=\mathbb{B}$ without explicit declaration of the mapping if it is unambiguous.
For convenience, we use $[N]\subseteq{}\mathbb{Z}^+$ to denote a set of positive integers from $1$ to $N$, where $\subseteq$ represents the subset relation.

A \emph{graph} $G=(\mathbb{V},\mathbb{E})$ consists of a \emph{vertex} set $\mathbb{V}$ and an \emph{edge} set $\mathbb{E}$.
For a graph $G$ of $N$ vertices, the set of its \emph{automorphisms}, written $Aut(G)$, is a collection of vertex permutations, under which the edges are preserved, and equivalent to a \emph{subgroup} of the symmetric group, \textit{i.e.}, $Aut(G)\leq{}\mathbb{S}_N$.
We call $H=(\mathbb{V}_H,\mathbb{E}_N)$ a \emph{(spanning) subgraph} of $G$ if $\mathbb{V}_H=\mathbb{V}$ and $\mathbb{E}_H\subseteq{}\mathbb{E}$.
Let $K_N=(\mathbb{V},\mathbb{E}_{K_N})$ be a \emph{complete graph} with $N$ vertices and $\mathbb{G}_N$ be the set containing all subgraphs of $K_N$.  We then know that any simple graphs of $N$ vertices are elements of $\mathbb{G}_N$.
The \emph{minimum and maximum degree} of a graph $G$ are denoted by $\delta$ and $\Delta$, respectively.

\subsection{Tensor and Tensor Networks (TNs)}
We consider an \emph{order-$N$ tensor} as a multi-dimensional array of real numbers represented by \mbox{$\mathcal{X}_{i_1,i_2,\ldots,i_N}\in{}
	\mathbb{R}^{I_1\times{}I_2\times\cdots\times{}I_N}$}, where 
the indices $i_n,\,n\in{}[N]$ correspond to the $\mathbb{R}^{I_n}$-associated \emph{tensor mode}.
Sometimes we ignore the indices by representing the same tensor as $\mathcal{X}$ for notational simplicity.
\emph{Tensor contraction} roughly refers to the process of summing over a pair of repeated indices between two tensors, which is though of as a natural extension of matrix multiplication into high-order tensors.
An explicit calculation of tensor contraction used in this paper follows the definition in~\cite{cichocki2016tensor}.

We consider \emph{tensor network}~(TN) as defined by~\citet{Ye2019Tensor}.
Suppose a sequence of vector spaces $\mathbb{R}^{I_i},\,i\in{}[N]$ and an edge-labelled simple graph $(G,r)=(\mathbb{V},\mathbb{E},r)$, where
$r:\mathbb{E}\rightarrow{}\mathbb{Z}^+$ represents the function labelling edges with  positive integers.
TN is thus intuitively defined as a set of tensors, whose elements are of the form of a sequence of tensor contraction of ``core tensors'' corresponding to vertices of $G$.
See \citep{Ye2019Tensor} for an explicit definition of a TN.
In the paper we refer to those core tensors as \emph{vertices}, to the unlabelled graph $G$ as TN \emph{format}, and to the function $r$ as  \emph{TN-(model)-ranks}.
Being consistent with \citep{Ye2019Tensor}, we use the same mathematical expression $TNS(G,r,\mathbb{R}^{I_1},\mathbb{R}^{I_2}\ldots,\mathbb{R}^{I_N})$ to represent a TN in our analysis.
The expression is also rewritten as $TNS(G,r)$ for shorthand if $\mathbb{R}^{I_n},\,n\in{}[N]$ are unimportant in the context.
Let $\mathbb{F}_G$ be the set consisting of all possible functions of $r$'s associated to $G$.
Then note that $\mathbb{F}_G$ is equivalent to a positive cone except zero of dimension $\vert{}\mathbb{E}\vert$, \textit{i.e.}, $\mathbb{F}_G=\mathbb{Z}^{+,\vert{}\mathbb{E}\vert}$.

\subsection{TN Structure Search~(TN-SS)}
Let $\mathcal{X}\in{}\mathbb{R}^{I_1\times{}I_2\times\cdots\times{}I_N}$ be an order-$N$ tensor.
TN-SS \emph{without noise} is to solve an optimization problem as follows:
\begin{equation}
\min_{r\in{}\mathbb{F}_{K_N}}\phi\left(K_N,r\right),\quad{}s.t.\,\mathcal{X}\in{}TNS(K_N,r),\label{eq:TNSS}
\end{equation}
where $\phi:\mathbb{G}_N\times{}\mathbb{F}_{K_N}\rightarrow{}\mathbb{R}$ represents a loss function measuring the model complexity of a TN.
Note that, although in~\eqref{eq:TNSS} the first term of $\phi$ is fixed to be $K_N$ , the TN format can degenerate into any simple graphs of $N$ vertices, as the edges of labeling with ``1'', \textit{i.e.}, $\{e\in\mathbb{E}_{K_N}\vert{}r(e)=1\}$, can be harmlessly discarded from the format~\citep{Ye2019Tensor,hashemizadeh2020adaptive}.
We see that solving~\eqref{eq:TNSS} is an \emph{integer programming} problem, generally NP-complete~\cite{papadimitriou1982complexity}.
Nevertheless, thanks to the fact $\mathbb{F}_{K_N}=\mathbb{Z}^{+,\vert\mathbb{E}_{K_N}\vert}$,
some practical algorithms have been proposed~\cite{hashemizadeh2020adaptive,kodryan2020mars,li2020evolutionary}, as $\mathbb{Z}^{+,\vert\mathbb{E}_{K_N}\vert}$ is a well-defined metric space with the isotropic property.
However, we will see next that such good properties do not hold for TN-PS anymore in general.

\section{Tensor-Network Permutation Search~(TN-PS)}\label{sec:counting}

In this section, we first make precise the problem of TN-PS and then prove the properties involving counting, metric, and neighborhood, which are crucial for both understanding the problem and deriving efficient algorithms.

\subsection{Problem Setup}\label{sec:ProblemSetup}
Recall the example illustrated in Figure~\ref{fig:cover}.
Suppose a tensor $\mathcal{X}$ of order $N$ and a simple graph $G_0$, , dubbed \emph{template}, of $N$ vertices.
Apart from the TN-ranks, the primary goal of TN-PS is to find the optimal mappings in some sense from the modes of $\mathcal{X}$ onto vertices of $G_0$.
We thus easily see that the problem amounts to searching the optimal \emph{permutation} of vertices of a graph.
More precisely,
solving TN-PS is to repeatedly index the vertices of $G_0$ consecutively from $1$ to $N$, and then to seek the optimal index sequence in some sense from all possibilities.
Since the permutations are bijective to each other, the TN structures arising from these permutations naturally form an equivalence class, of which all elements preserve the same ``diagram'' as $G_0$.
Formally, such the equivalence class to the template $G_0=(\mathbb{V},\mathbb{E}_0)$ can be written as follows:
\begin{equation}
\mathbb{G}_0=\left\{G\in{}\mathbb{G}_N\vert{}G\cong{}G_0\right\},\label{eq:G0}
\end{equation}
where $\cong$ denotes the relation of \emph{graph isomorphism}, meaning that for each $G\in{}\mathbb{G}_0$ there exists a vertex permutation $g_G\in\mathbb{S}_N$ such that $G=(g_G\left(\mathbb{V}\right),\mathbb{E}_0)$ holds, or $G=g_G\cdot{}G_0$ for shorthand.
TN-PS (without noise) is thus defined by restricting the search space of~\eqref{eq:TNSS} to $\mathbb{G}_0$ as follows:
\begin{equation}
\min_{(G,r)\in{}\mathbb{G}_0\times\mathbb{F}_{G_0}}\phi\left(G,r\right),\quad{}s.t.\,\mathcal{X}\in{}TNS(G,r).\label{eq:TNPS}
\end{equation}
Compared to TN-SS, we search TN structures from a new space consisting of two ingredients: a non-trivial graph set $\mathbb{G}_0$ and $\mathbb{F}_{G_0}=\mathbb{Z}^{+,\vert\mathbb{E}_0\vert}$ that corresponds to the TN-ranks.
We see that TN-PS is no longer an integer programming problem as TN-SS due to the irregular geometry of $\mathbb{G}_0$.
Meanwhile, the size of the new search space varies with different template $G_0$.
Figure~\ref{fig:space} visualizes intuitively the ``geometrical shape'' of the search space for TN-PS associated to a template of three vertices and two edges.
We see that the search space of TN-PS is more ``collapsed'' than the original TN-SS.
One immediate consequence of collapsing is that the searching path and solutions for TN-SS would run out of the TN-PS region, thereby failing to preserve the original TN format. 
Next, we will establish formal statements for these observations, and the results will help develop feasible algorithms for resolving TN-PS.

\begin{figure}
	\centering
	\includegraphics[width=0.8\columnwidth]{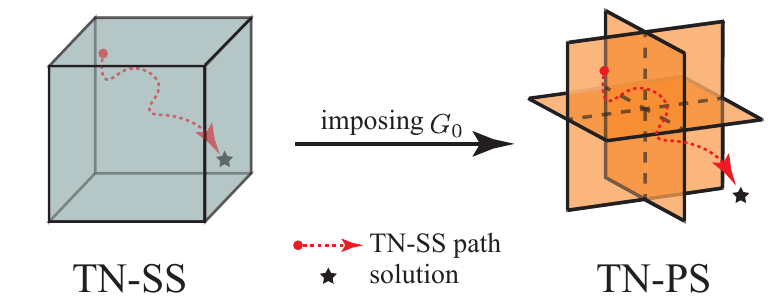}
	\caption{``Geometrical shape'' of search spaces of TN-SS and TN-PS, where the equivalence class $\mathbb{G}_0$ makes the ``shape'' for TN-PS as a combination of flips of low-dimensional spaces. 
	}\vskip -0.1in
	\label{fig:space}
\end{figure}

\subsection{Counting TN Structures}
We begin by counting the size of the new search space, proving that 
the graph degrees of the template $G_0$ give a universal bound for the size of the search space of TN-PS.

Suppose first a simple graph $G_0=(\mathbb{V},\mathbb{E}_0)$ of $N$ vertices as the template, by which we then construct the set $\mathbb{G}_0$ as Eq.~\eqref{eq:G0}.
As mentioned above, we have known two facts:
\emph{1)} $Aut(G_0)$ forms a subgroup of $\mathbb{S}_N$, \textit{i.e.}, $Aut(G_0)\leq{}\mathbb{S}_N$, such that $G_0=a\cdot{}G_0$ for any $a\in{}Aut(G_0)$;
and \emph{2)} for every $G\in\mathbb{G}_0$ there exists $g_G\in\mathbb{S}_N$ such that $G=g_G\cdot{}G_0$.
By these facts, $G=g^\prime\cdot{}G_0$ holds for any $g^\prime=g_G\cdot{}a$, implying that for each $G\in\mathbb{G}_0$ there exists a \emph{left coset} of $Aut(G_0)$, which is of the form \mbox{$g_G\cdot{}Aut(G_0):=\{g_G\cdot{}a\vert{}a\in{}Aut(G_0)\}$}.
According to the \emph{Lagrange's theorem} in group theory, we thus obtain the following equation with respect to the size of $\mathbb{G}_0$:
\begin{equation}
\vert{}\mathbb{S}_N\vert=\vert\mathbb{G}_0\vert\cdot\vert{}Aut(G_0)\vert.\label{eq:Lag}
\end{equation}
Table~\ref{tab:TNGraphs} lists the values of $\vert{}Aut(G_0)\vert$ associated with several commonly used TNs.
The size of $\mathbb{G}_0$ for those TNs can be therefore derived by~\eqref{eq:Lag}, shown in the last row of Table~\ref{tab:TNGraphs}.
\begin{table}[t]\small
	\caption{
		Illustration of several counting-related properties for commonly used TNs of order $N>3$, including \textit{tensor train} (TT, \citealt{oseledets2011tensor}), \textit{tensor tree} (TTree, \citealt{Ye2019Tensor}), TR and \textit{projected entangled pair states}~(PEPS, \citealt{verstraete2004renormalization}), where $G_0=(\mathbb{V},\mathbb{E}_0)$, and $\delta_0$ and $\Delta_0$ denote the minimum and maximum degree of $G_0$, respectively.
	}
	\vspace{0.3cm}
	\centering
	\renewcommand{\arraystretch}{1.2}
	\begin{tabular}{lcccc}\toprule
		& \textbf{TT} & \textbf{TTree} & \textbf{TR} & \textbf{PEPS}\\\midrule
		$G_0$ & Path & Tree & Cycle & Lattice \\
		$\delta_{0}$& $1$ & $1$ & $2$ & $2$\\
		$\Delta_{0}$& $2$ & $[2,N-1]$ & $2$ & $2,3,4$\\
		$\vert{}\mathbb{E}_0\vert$& $N-1$ & $N-1$ & $N$ & $\leq{}N$ \\
		{\scriptsize$\vert{}Aut(G_0)\vert$}& $2$ & $[2,(N-1)!]$ & $2N$ & $\leq{}N$\\
		$\vert\mathbb{G}_0\vert$& $N!/2$ & $[N,N!/2]$ & {\tiny$(N-1)!/2$} & $\leq{}N!/4$\\
		
		\bottomrule
	\end{tabular}
	\vskip -0.1in
	\label{tab:TNGraphs}
\end{table}

However, counting the automorphisms for a general graph is difficult~\citep{chang1995finding}. 
Blow we prove that the size of the search space of TN-PS is controlled by the the minimum and maximum degree of $G_0$.
For convenience, we further assume
that TN-ranks are only searched within a finite range $\mathbb{F}_{G_0,R}\subset{}\mathbb{F}_{G_0}$, meaning that the rank $r(e)\leq{}R$ holds for any $r\in{}\mathbb{F}_{G_0,R}$ and $e\in{}\mathbb{E}_0$.
We then have the following counting bounds.
\begin{theorem}\label{thm:OverallBounds}
	Assume $G_0$ to be a simple and connected graph of $N$ vertices, and $\mathbb{G}_0$ is constructed as~\eqref{eq:G0}.
	Let $\delta=N/d_1$ and $\Delta=N/d_2$, $d_1\geq{}d_2>1$, be the minimum and maximum degree of $G_0$, respectively.
	The size of the search space of~\eqref{eq:TNPS}, written $\mathbb{L}_{G_0,R}:=\mathbb{G}_0\times{}\mathbb{F}_{G_0,R}$, is bounded as follows:
	\begin{equation}
	\begin{split}
	R^{\frac{N^2}{2d_2}}\cdot{}N!\geq\vert{}\mathbb{L}_{G_0,R}\vert\geq{}
	R^{\frac{N^2}{2d_1}}\cdot{}e^{\gamma(d_2)\cdot{}N-\frac{1}{2}\log{}d_2-1/24}
	\end{split},\label{eq:SizeBound}
	\end{equation}
	where $\gamma(d)=\log{}d+\frac{1}{d}-1$ is a positive and monotonically increasing function for $d>1$.
\end{theorem}
Proving the above theorem requires the following lemma about an upper-bound of the size of $Aut(G_0)$, of which the proof is given in Appendix~\ref{sec:Proofs}. 
\begin{lemma}\label{thm:CountingLemma}
	Let $G_0$ be a simple graph of $N$ vertices, and $Aut(G_0)$ be the set containing automorphisms of $G_0$.
	Assume that $G_0$ is connected and its maximum degree $\Delta$ satisfies $N/\Delta=d>1$, then the inequality
	\begin{equation}
	\left\vert{}Aut(G_0)\right\vert\leq{}N!\cdot{}e^{-\gamma(d)\cdot{}N+\frac{1}{2}\log{}d+1/24}\label{eq:CountingAutExplict}
	\end{equation}
	holds, where $\gamma(\,\cdot\,)$ is defined in Theorem~\ref{thm:OverallBounds}. 
\end{lemma}
As shown in~\eqref{eq:SizeBound}, the bounds of $\vert\mathbb{L}_{G_0,R}\vert$ are determined by three factors: the number of vertices $N$, the searching range of TN-ranks $R$, and the graph degrees of $G_0$ parameterized by $d_1$ and $d_2$.
Figure~\ref{fig:Bound} shows the bounds in ~\eqref{eq:SizeBound} with varying these factors.
We see from the left panel that the upper and lower bounds go closer with increasing the value of $d$~(where we assume $d_1=d_2=d$ for brevity).
It implies that the bounds are tight for graphs with small degrees. 
We also see from the middle panel that $\vert\mathbb{L}_{G_0,R}\vert$ grows fast with $N$, even though the graph degree $d$ been sufficiently small such as in TT/TR, while the growth is relatively slow with increasing $R$, the search range for TN-ranks.

\begin{figure}
	\centering
	\includegraphics[width=1\columnwidth]{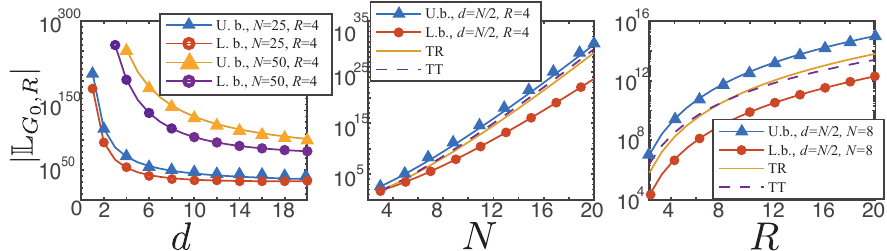}
	\vspace{-0.4cm}
	\caption{Illustration of bounds given in Theorem~\ref{thm:OverallBounds} with varying the parameters $d$, $N$ and $R$, where ``U.b.'' and ``L.b.'' denote the upper and lower bounds, respectively, and $d_1=d_2=d$.
	}\vskip -0.1in
	\label{fig:Bound}
\end{figure}

\subsection{Semi-Metric and Neighborhood}
The notion of metric and neighborhood of the search space are fundamental for most \emph{steepest-descent}-based optimization methods.
We have seen that they are well-defined for TN-SS but remain unknown for TN-PS.
To address the issue, we establish below a new (semi-)metric and neighborhood for TN-PS with rigorous proofs using \emph{group-theoretic} instruments.
The application of these results to developing efficient algorithms for TN-PS will be introduced in the next section.

We begin by establishing the semi-metric, a relaxation of metric satisfying separation and symmetry except possibly for the triangle inequality, over the graph set $\mathbb{G}_0$.
Although there has been much literature in which different definitions of the graph metric or similarity are proposed, most of them are computationally hard~\cite{koutra2011algorithms}.
Unlike those works, we construct the semi-metric over $\mathbb{G}_0$ based on the equivalence property of its elements given in~\eqref{eq:G0}, so it can be built up by graph isomorphisms in a simple fashion.

Recall the symmetric group $\mathbb{S}_N$.
Let $\mathbb{T}_N\subseteq{}\mathbb{S}_N$ be the set consisting of its all adjacent transpositions, the operations of swapping adjacent two integers in $[N]$ and fixing all other integers.
We thus know from group theory that $\mathbb{T}_N$ generates $\mathbb{S}_N$.
Furthermore, 
let $d_{\mathbb{T}_N}:\mathbb{S}_N\times{}\mathbb{S}_N\rightarrow{}R$ be the \emph{word metric}~\citep{luck2008survey} of $\mathbb{S}_N$ induced by $\mathbb{T}_N$.
Intuitively, the value of $d_{\mathbb{T}_N}(p_1,p_2),\,p_1,p_2\in\mathbb{S}_N$ reflects the minimum number of adjacent swapping operations required for transforming the permutation from $p_1$ to $p_2$.
Since we saw in Section~\ref{sec:ProblemSetup} that for each $G\in\mathbb{G}_0$ there is a permutation $g_G\in{}\mathbb{S}_N$ such that $G=g_G\cdot{}G_0$,  we thus construct a function $d_{G_0}:\mathbb{G}_0\times{}\mathbb{G}_0\rightarrow{}\mathbb{R}$ using the word metric $d_{\mathbb{T}_N}$ as follows:
\begin{equation}
d_{G_0}\left(G_1,G_2\right)=\min_{
	p_i\in{}g_i\cdot{}Aut(G_0),i=1,2
}d_{\mathbb{T}_N }(p_1,p_2),\label{eq:metric}
\end{equation}
where $G_1,G_2\in\mathbb{G}_0$ and $g_1,g_2\in\mathbb{S}_N$ are permutations satisfying $G_i=g_i\cdot{}G_0$, $i=1,2$.
The following lemma shows that \eqref{eq:metric} is in fact a semi-metric function, followed by the construction of the corresponding neighborhood in $\mathbb{G}_0$.
\begin{lemma}\label{thm:metric}
	Let $G_0$ be a simple graph and $\mathbb{G}_0$ be the set defined as~\eqref{eq:G0}.
	The function $d_{G_0}:\mathbb{G}_0\times{}\mathbb{G}_0\rightarrow{}\mathbb{R}$ defined by \eqref{eq:metric} is a semi-metric on $\mathbb{G}_0$.
	Furthermore, let $\mathbb{I}_d\left(G\right)$ be a set constructed as follows:
	\begin{equation}
	\begin{split}
	\mathbb{I}_d\left(G\right)=
	\{G'&\in\mathbb{G}_0\vert{}G'=q\prod_{i=1}^d{}t_i\cdot{}G_0,\\
	&q\in{}g\cdot{}Aut(G_0),\,t_i\in{}\mathbb{T}_N,\,i\in[d]\}
	\end{split}.\label{eq:nh}
	\end{equation}
	Then $\mathbb{N}_D\left(G\right)=\bigcup_{d=0}^D\mathbb{I}_d(G)$ is the neighborhood of \mbox{$G=g\cdot{}G_0\in{}\mathbb{G}_0$} induced by~\eqref{eq:metric}, with the radius 
	$D\in{}\mathbb{Z}^+\cup\{0\}$.
\end{lemma}
\begin{algorithm}[tb]
	\caption{Random sampling over $\mathbb{I}_d(G)$}
	\label{alg:sampling}
	\begin{algorithmic}
		\STATE {\bfseries Input:}
		Center: $G\in\mathbb{G}_0$ with $N$ vertices; Radius: $d$.
		\STATE {\bfseries Initialize:}
		$G^\prime=G$ where $G^\prime=(\mathbb{V}^\prime,\mathbb{E}^\prime)$.
		\FOR{$k=1$ {\bfseries to} $d$}
		\STATE Uniformly draw $i,j\in{}[N],\,i\neq{}j$ in random.
		\STATE Choose $v_i,v_j\in\mathbb{V}^\prime$ and swap them.
		\ENDFOR
		\STATE {\bfseries Output:} $G^\prime$.
	\end{algorithmic}
\end{algorithm}
We see from~\eqref{eq:nh} that $\mathbb{N}_D\left(G\right)$ consists of combinations of two sets: $Aut(G_0)$ and $\mathbb{T}_N$, followed by the permutation representative $g$ associated to the center graph $G$.
It thus suggests a straightforward sampling method over $\mathbb{N}_D\left(G\right)$, that is, combinatorially sampling over $Aut(G_0)$ and $\mathbb{T}_N$ from some distributions.
However, obtaining all elements of $Aut(G_0)$ is computationally hard~(NP-intermediate,~\citealt{goldwasser1989knowledge}) in general.
To avoid this, we prove that sampling using Alg.~\ref{alg:sampling} can cover all elements of $\mathbb{I}_d(G)$ without sampling directly over $Aut(G_0)$.

\begin{theorem}\label{thm:cover}
	For every $G^\prime\in{}\mathbb{I}_d(G)$ with $G\in\mathbb{G}_0$ and $d\geq{}1$, the probability that the output of Alg.~\ref{alg:sampling} equals $G^\prime$ is positive.
\end{theorem}
The (semi-)metric and neighborhood for the overall search space of TN-PS, \textit{i.e.}, $\mathbb{G}_0\times{}\mathbb{F}_{G_0,R}$, can be thus derived by composing the Euclidean metric of $\mathbb{F}_{G_0,R}\subseteq{}\mathbb{Z}^{+,\vert\mathbb{E}_0\vert}$.
In the next section, Alg.~\ref{alg:sampling} will be applied to the new algorithm, by which the searching efficiency is significantly improved.

\section{Meta-Heuristic via Local Sampling}
We present now a new meta-heuristic algorithm for searching TN structures.
Unlike the existing methods such as~\cite{li2020evolutionary}, we exploit the information of the ``steepest-descent'' direction, estimated by sampling over a neighborhood of the search space, to accelerate the searching procedure.

Suppose a tensor $\mathcal{X}$ of order $N$.
For the practical purpose, we take the influence of noise into~\eqref{eq:TNPS}, which is given by
\begin{equation}
\begin{split}
&\min_{G,r,\mathcal{Z}}\phi\left(G,r\right)+\lambda\cdot{}RSE(\mathcal{X},\mathcal{Z})\\
s.t.&\,(G,r)\in\mathbb{G}_0\times{}\mathbb{F}_{G_0,R},\,\mbox{and }\mathcal{Z}\in{}TNS(G,r)\label{eq:PTNPS}
\end{split},
\end{equation}
where $\lambda>0$ denotes a tuning parameter associated with the noise variance, and $RSE$ is the function of \emph{relative squared error}~(RSE) for modeling the influence of Gaussian noise.
In other applications such as in generative models~\cite{liu2021tensor},
it can be replaced by KL-divergence without modifying the algorithm details.
The searching algorithm is illustrated in Alg.~\ref{alg:main}, where $N_{[R]}(a,b)$ denotes a Gaussian distribution of the mean $a$ and variance $b$, followed by a truncation operation such that the samples out of the range $[R]$ are pulled back to the closest bound, $Ber(p)$ denotes the Bernoulli distribution of the mean $p$, and $\mathbf{I}$ denotes the identity matrix of dimension $\vert\mathbb{E}_0\vert\times{}\vert\mathbb{E}_0\vert$.

In each iteration, we elaborate the algorithm into three phases: \emph{local-sampling, evaluation}, and \emph{updating}.
Suppose the starting point $(G^{\{m\}},r^{\{m\}})$ at the $m$th iteration.
In the local-sampling phase, we randomly draw samples for both the TN-ranks ($s_k$) and the ``mode-vertex'' maps ($H_k$) over the neighborhood centered at $(G^{\{m\}},r^{\{m\}})$.
The aim is to explore good descent directions within the neighborhood.
The sampling distributions, involving the rounded truncated Gaussian and the uniform distribution in Alg.~\ref{alg:sampling}, are chosen as a non-informative searching prior.
Note that for the two variance-related parameters $c_1,c_2\in[0.9,1)$, we apply the annealing trick to the shrinkage of the sampling range in each iteration.
The trick guarantees the convergence of the algorithm.
In the evaluation phase, we employ \emph{arbitrarily} proper optimization or learning methods to minimize RSE or other alternatives for the sampled structures $(H_k,s_k),\,k\in{}[\#Sample]$.
In the updating phase, we calculate the overall loss function $f_k$ for each sampled structures, and then update $(G^{\{m+1\}},r^{\{m+1\}})$ once there exists new samples, whose performance is better than $(G^{\{m\}},r^{\{m\}})$.
More precisely, let $f_0$ be the loss of~\eqref{eq:PTNPS} with respect to $(G^{\{m\}},r^{\{m\}})$ and $f_{min}$ be the minimum among all $f_k$'s.
If $f_{min}<f_0$, we update $(G^{\{m+1\}},r^{\{m+1\}})=(H_{min},s_{min})$; otherwise, we remain $(G^{\{m+1\}},r^{\{m+1\}})=(G^{\{m\}},r^{\{m\}})$.

\noindent\textbf{Discussion.}
Compared with the global-sampling methods, such as genetic algorithm~(GA)~\cite{hayashi2019exploring,li2020evolutionary}, we restrict the sampling range into the neighborhoods of the structure candidates, rather than the whole search space.
The advantages in doing so are mainly two-folds:
first, the neighborhood geometry allows to construct gradient-like directions, which result in a faster decrease of the loss function if the landscape is related smooth;
second, a smaller sampling range can mitigate the curse of dimensionality.
Otherwise, the algorithm would ``lose itself'' in searching if the TN structure is large scale.
However, it should be mentioned that the local-sampling methods would perform worse than the global-searching ones if the landscape is too ``flat'' or ``swinging''.
We conjecture with rich empirical observations that TN-PS (including TN-SS) seems more suitable for ``local-sampling'' methods.
Since a small perturbation on the structure such as ranks would not dramatically change the RSE, it implies that the landscape of~\eqref{eq:PTNPS} tends to be smooth.
Although a rigorous discussion on this issue remains open, we use extensive numerical results to verify the efficiency of Alg.~\ref{alg:main} in the next section.

\begin{algorithm}[tb]
	\caption{TN-structure Local Sampling (TNLS)}
	\label{alg:main}
	\begin{algorithmic}
		\STATE {\bfseries Input:} template graph: $G_0=(\mathbb{V},\mathbb{E}_0)$; searching range of TN-ranks: $R$; function: \mbox{$f(G,r,\mathcal{Z}):=\phi\left(G,r\right)+\lambda\cdot{}RSE(\mathcal{X},\mathcal{Z})$}, maximum iteration: $\#Iter>0$; number of sampling: $\#Sample>0$; tuning parameters: $c_1,c_2\in[0.9,1)$.
		\STATE {\bfseries 0. Initialize:}\\
		$(G^{\{0\}},r^{\{0\}})$ with random selection from $\mathbb{G}_0\times{}\mathbb{F}_{G_0,R}$.
		\STATE Obtain $\mathcal{Z}^{\{0\}}$ by arbitrary TN approximation methods with $(G^{\{0\}},r^{\{0\}})$.
		\FOR{$m=1$ {\bfseries to} $\#Iter$}
		\STATE {\bfseries 1. Local sampling:}
		\FOR{$k=1$ {\bfseries to} $\#Sample$}
		\STATE Sample $s_k\sim{}N_{[R]}(r^{\{m\}},c_1^{m-1}\cdot\mathbf{I})$ with rounding.
		\IF{\texttt{TRUE} by sampling from $Ber(c_2^{m-1})$}
		\STATE Sample $H_k\in\mathbb{G}_0$ from $\mathbb{I}_1(G^{\{m\}})$ using Alg.~\ref{alg:sampling}.
		\ELSE 
		\STATE $H_k=G^{\{m\}}$.
		\ENDIF
		\ENDFOR
		\STATE { \textbf{2. Evaluation:} \emph{(be possible in parallel)}}
		\FOR{$k=1$ {\bfseries to} $\#Sample$}
		\STATE Obtain $\mathcal{Z}_k$ by arbitrary TN approximation methods with the given $(H_k,s_k)$.
		\ENDFOR 
		\STATE {\bfseries 3. Update:}
		\begin{equation}\nonumber
		\begin{split}
		(G&^{\{m+1\}},r^{\{m+1\}},\mathcal{Z}^{\{m+1\}})=\\
		\arg&\min_{k}\{f(G,r,\mathcal{Z})\vert{}(G,r,\mathcal{Z})=(G^{\{m\}},r^{\{m\}},\mathcal{Z}^{\{m\}}),\\
		&(G,r,\mathcal{Z})=(H_k,s_k,\mathcal{Z}_k)\}.
		\end{split}
		\end{equation}
		\ENDFOR\\
		{\bfseries Return:} $(G^{\{m+1\}},r^{\{m+1\}})$.
	\end{algorithmic}
\end{algorithm}

\section{Experimental Results}
In this section, we numerically verify the effectiveness and efficiency of the proposed method on both the synthetic and real-world tensors.
\begin{table*}[t]
	\centering
	\caption{Experimental results on synthetic data in TT/TR format.
		In the table, \emph{Eff.} denotes the parameter efficiency defined in~\cite{li2020evolutionary},
		$\#Eva.$ denotes the number of evaluations, and ``-'' denotes that the methods fail to satisfy the condition $RSE\leq{}10^{-4}$.
	}
	\begin{threeparttable}\small\label{tab:TR}
		\setlength{\tabcolsep}{1.7mm}{   	
			\begin{tabular}{ccccc|cc|cc}
				\toprule
				\multirow{4}[0]{*}{\textbf{Trial}}&\multicolumn{4}{c}{\emph{TR methods}}&\multicolumn{2}{c}{\emph{TN-SS methods}}&\multicolumn{2}{c}{\emph{TN-PS methods}}\\
				& \textbf{TR-SVD} &  \textbf{TR-LM} & \textbf{TR-ALSAR} & \textbf{Bayes-TR} & \textbf{Greedy}&\textbf{TNGA}$^\star$&\textbf{TNGA+}&\textbf{TNLS}\\
				\cmidrule{2-9}
				{}&\multicolumn{8}{c}{\textbf{ Order 4} -- \emph{Eff.$\uparrow$~$\langle$\emph{Is TT/TR format preserved? {\color{green}Yes} or {\color{red}No}.}$\rangle$~$[$\#Eva.$\downarrow$$]$} }\\
				\midrule
				\textbf{A} & 1.00 & 1.00 & 0.21 &  1.00&1.00~$\langle${\color{green}Yes}$\rangle$ &1.00~$\langle${\color{green}Yes}$\rangle$~[600] & 1.00~[450]&1.00~[\textbf{361}]  \\
				\textbf{B} & 0.64& 1.00 & 1.00 &  0.64 &0.89$\langle${\color{red}No}$\rangle$& 1.00~$\langle${\color{green}Yes}$\rangle$~[300] &1.00~[450] &1.00~[\textbf{241}] \\
				\textbf{C} & 1.17 & 1.17 &0.23 &  1.00&1.17~$\langle${\color{green}Yes}$\rangle$ &1.17~$\langle${\color{green}Yes}$\rangle$~[750] & 1.17~[450]&1.17~[\textbf{181}]  \\
				\textbf{D} & 0.57 & 0.57 & 0.32 &  -&1.00~$\langle${\color{green}Yes}$\rangle$ &1.00~$\langle${\color{green}Yes}$\rangle$~[450] &1.00~[300]&1.00~[\textbf{301}]   \\
				\textbf{E} & 0.43& 0.48 & 0.40 &  0.40&1.00~$\langle${\color{green}Yes}$\rangle$ &1.00~$\langle${\color{green}Yes}$\rangle$~[1050] & 1.00~[450] &1.00~[\textbf{361}] \\
				
				\midrule
				
				{}&\multicolumn{8}{c}{\textbf{ Order 6} -- \emph{Eff.$\uparrow$~$\langle$\emph{Is TT/TR format preserved? {\color{green}Yes} or {\color{red}No}.}$\rangle$~$[$\#Eva.$\downarrow$$]$} }\\
				\midrule
				\textbf{A}&0.21&0.44&-&-&0.16~$\langle${\color{red}No}$\rangle$&0.82~$\langle${\color{red}No}$\rangle$~[1650] &\textbf{1.00}~[1500]&\textbf{{1.00}}~[\textbf{661}]\\
				\textbf{B}&0.14&0.15&0.14&-&0.27$~\langle${\color{red}No}$\rangle$&- &\textbf{1.00}~[1350]&\textbf{1.00}~[\textbf{601}]\\
				\textbf{C}&0.57&1.00&0.85&0.29&0.97~$\langle${\color{red}No}$\rangle$&1.00~$\langle${\color{green}Yes}$\rangle$~[3300] &1.00~[1800]&1.00~[\textbf{661}]\\
				\textbf{D}&0.21&0.39&0.10&0.13&1.04~$\langle${\color{green}Yes}$\rangle$&1.04~$\langle${\color{green}Yes}$\rangle$~[2700] &\textbf{1.16}~[1500]&\textbf{1.16}~[\textbf{601}]\\
				\textbf{E}&0.15&0.30&-&0.12&1.00~$\langle${\color{green}Yes}$\rangle$&1.00~$\langle${\color{green}Yes}$\rangle$~[2400] &1.00~[1050]&1.00~[\textbf{541}]\\

				\midrule
				{}&\multicolumn{8}{c}{\textbf{ Order 8} -- \emph{Eff.$\uparrow$~$\langle$\emph{Is TT/TR format preserved? {\color{green}Yes} or {\color{red}No}.}$\rangle$~$[$\#Eva.$\downarrow$$]$} }\\
				\midrule
				\textbf{A}&0.10&0.16&-&0.03&0.88~$\langle${\color{red}No}$\rangle$&0.48~$\langle${\color{red}No}$\rangle$~[2550]&\textbf{1.00}~[2850]&\textbf{1.00}~[\textbf{1021}]\\
				\textbf{B}&0.09&0.43&-&-&0.61$\langle${\color{red}No}$\rangle$&-&\textbf{1.02}~[2250]&\textbf{1.02}~[\textbf{961}]\\
				\textbf{C}&0.03&0.31&-&0.02&\textbf{1.16}$\langle${\color{red}No}$\rangle$&0.49~$\langle${\color{red}No}$\rangle$~[2250]&1.11~[3750]&1.11~[\textbf{1321}]\\
				\textbf{D}&0.20&0.53&-&-&1.03$\langle${\color{red}No}$\rangle$&0.32~$\langle${\color{red}No}$\rangle$~[4050]&\textbf{1.06}~[1950]&\textbf{1.06}~[\textbf{781}]\\
				\textbf{E}&0.33&0.33&-&-&\textbf{1.17}~$\langle${\color{green}Yes}$\rangle$&0.23~$\langle${\color{red}No}$\rangle$~[1500]&0.88~[1500]&\textbf{1.17}~[\textbf{901}]\\

				\bottomrule
			\end{tabular}
		}
	\end{threeparttable}
\end{table*}

\subsection{Synthetic Data in TT/TR Format and Beyond}
Using synthetic data, we first verify:
\textbf{(a)} TN-PS can reduce the required TN model size for the low-rank tensor approximation task, reflecting the improvement of the expressive power of TNs;
and \textbf{(b)} the proposed local-sampling method achieves more efficient searching than the existing sampling-based methods.

\begin{figure}
	\centering
	\includegraphics[width=1\columnwidth]{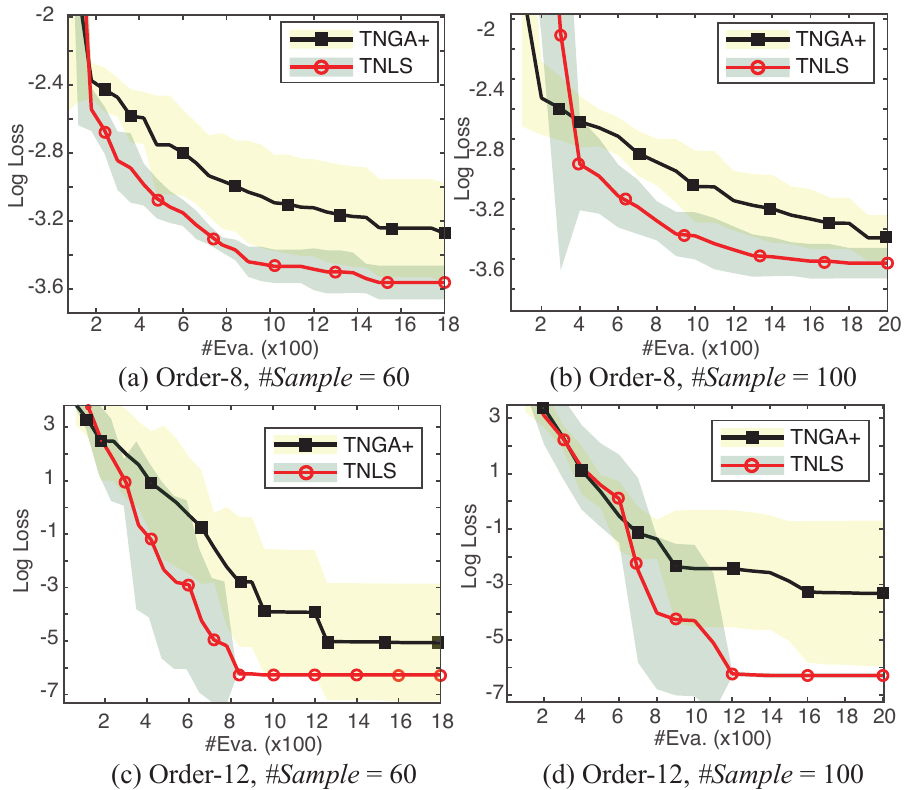}
	\vspace{-0.6cm}
	\caption{Average loss with varying the number of evaluations.
	}\vskip -0.1in
	\label{fig:convergence}
\end{figure}

\noindent\textbf{Data generation.}
We choose TT/TR, the most commonly used TN formats in machine learning, to generate tensor data.
For each tensor order, \textit{i.e.}, the number of vertices, $N\in\{4,6,8\}$, we generate five tensors by randomly choosing ranks and values of vertices (core tensor).
In more detail, the dimensions for each tensor modes are set to equal $3$.
Here we choose a small dimension as same as the one in~\citep{li2020evolutionary} because it is typically irrelevant to the searching difficulty, as shown in Theorem~\ref{thm:OverallBounds}.
Meanwhile, we uniformly select the TN-ranks from $\{1,2,3,4\}$ in random, and \textit{i.i.d.} draw samples from Gaussian distribution $N(0,1)$ as the values of vertices.
After contracting all vertices, we finally uniformly permute the tensor modes in random.
The permutations maintain unknown for all algorithms in the experiment.

\noindent\textbf{Experiment setup.}
In our method, we set the template $G_0$ as a cycle graph, the searching range for TN-ranks $R=7$, the maximum iteration $\#Iter=30$, the number of samples $\#Sample=60$, and the tuning parameters $c_1=0.9$ and $c_2=0.9, 0.94, 0.98$ for three different tensor orders $N=4,6,8$, respectively.

In the experiment, we also implement various TR decomposition methods with adaptive rank selection for comparison, including TR-SVD and TR-ALSAR~\citep{zhao2016tensor}, Bayes-TR~\citep{tao2020bayesian}, and TR-LM~\citep{mickelin2020algorithms}.
We also compare our method with two SOTAs for TN-SS, including ``Greedy''~\cite{hashemizadeh2020adaptive} and TNGA~\citep{li2020evolutionary}.
Note that the original TNGA is forced to search only the TN formats.
For a fair comparison, we extend it into two new versions: in ``TNGA$^\star$'' we trivially allow the method to search the formats and ranks simultaneously; and in ``TNGA+'' we use the classic ``random-key'' trick~\cite{bean1994genetic} to encode $\mathbb{G}_0$ into chromosomes, such that TNGA+ is capable of solving TN-PS as well but remaining the same genetic operations as TNGA.
In these two methods and ours, we set the function $\phi$ in~\eqref{eq:PTNPS} to be the compression ratio defined in~\cite{li2020evolutionary}, the parameter $\lambda=200$, and apply the ``Adam'' optimizer~\cite{kingma2014adam} with the same parameters in the evaluation phase of Alg.~\ref{alg:main}.

We force all methods to achieve $RSE\leq{}10^{-4}$ for each tensor, a pretty small approximation error, in the experiment.
Otherwise we say the approximation fails.
For performance evaluation, we use the \emph{Eff.} index~\citep{li2020evolutionary}, the ratio of parameter number of TNs between the searched structure and the one in data generation, as the main performance measure.
A larger value of \emph{Eff.} implies using fewer parameters to achieve the close approximation error.
For the sampling-based methods, we also report the total number of evaluations required to obtain the solution, shown as $[\#Eva.]$ in Table~\ref{tab:TR}, reflecting the computational cost for those methods.

\begin{table}[t]
	\centering
	\caption{
		Experimental results on synthetic data in various TN formats.
	}
	\begin{threeparttable}\small\label{tab:TN}
		\setlength{\tabcolsep}{1.5mm}{   	
			\begin{tabular}{llcccc}
				\toprule
				\multirow{2}[0]{*}{\textbf{TNs}}&\multirow{2}[0]{*}{\textbf{Methods}}&\multicolumn{4}{c}{\textbf{Trial} -- \emph{Eff.$\uparrow$}~$\langle$\emph{formats preserved? {\color{green}Y} or {\color{red}N}.}$\rangle$ }\\
				\cmidrule{3-6}&
				& \textbf{A} & \textbf{B} & \textbf{C}  & \textbf{D} \\
				\midrule
				
				\multirow{3}[0]{*}{\textbf{TTree}}
				&TNGA+&1.29&1.17&1.11&1.55
				\\
				&TNLS&1.29&1.17&1.11&1.55
				\\
				&Greedy&1.29~$\langle${\color{red}N}$\rangle$&1.03~$\langle${\color{red}N}$\rangle$&0.79~$\langle${\color{red}N}$\rangle$&1.27~$\langle${\color{red}N}$\rangle$
				
				\\
				\cmidrule{2-6}
				
				\multirow{3}[0]{*}{\textbf{PEPS}}
				&TNGA+&1.14&1.00&1.00&1.21
				\\
				&TNLS&1.14&1.00&1.00&1.21
				\\
				&Greedy&1.03~$\langle${\color{red}N}$\rangle$&1.00~$\langle${\color{green}Y}$\rangle$&1.13~$\langle${\color{green}Y}$\rangle$&1.21~$\langle${\color{green}Y}$\rangle$
				\\
				
				\cmidrule{2-6}
				
				\multirow{3}[0]{*}{\textbf{HT}}
				&TNGA+&1.45&1.21&1.18&1.29
				\\
				&TNLS&1.45&1.21&1.18&1.29
				\\
				&Greedy&1.81~$\langle${\color{red}N}$\rangle$&1.91~$\langle${\color{red}N}$\rangle$&1.73~$\langle${\color{red}N}$\rangle$&2.16~$\langle${\color{red}N}$\rangle$
				\\
				\cmidrule{2-6}
				
				\multirow{3}[0]{*}{\textbf{MERA}}
				&TNGA+&0.95&1.32&2.30&1.00~
				\\
				&TNLS&1.09&1.88&2.88&1.03
				\\
				&Greedy&-&1.65~$\langle${\color{red}N}$\rangle$&0.78~$\langle${\color{red}N}$\rangle$&0.88~$\langle${\color{red}N}$\rangle$
				\\
				\bottomrule
				
			\end{tabular}
		}
	\end{threeparttable}
	\vskip -0.1in
\end{table}
\begin{table}[t]
	\caption{Average RSE$\downarrow$ for \emph{image completion} by methods TRALS~\citep{wang2017efficient}, Greedy, TNGA$+$, and TNLS (ours), where the values in square brackets are the number of evaluations required by methods.
	}
	\vspace{0.3cm}
	\centering
	\begin{threeparttable}\centering\small
		
		\begin{tabular}{ccccc}
			\toprule

			& \textbf{TRALS} & \textbf{Greedy} & \textbf{TNGA+}   & \textbf{TNLS} \\
			\midrule
			
			$90\%$&0.192&0.186&\textbf{0.178}~[4929]&\textbf{0.178}~[\textbf{1544}]\\
			$70\%$&0.142&0.142&\textbf{0.132}~[6086]&0.136~[\textbf{1437}]\\
			\bottomrule
			
		\end{tabular}
	\end{threeparttable}\label{tab:completion}
	\vskip -0.1in
\end{table}

\noindent\textbf{Results.}
The experimental results are reported in Table~\ref{tab:TR}, where the symbol ``$-$'' denotes that the method fails to achieve the condition $RSE\leq{}10^{-4}$, and the ``$\langle${\color{green}Yes} or {\color{red}No}$\rangle$'' denotes whether the structures obtained by methods preserve the TT/TR format or not.
As shown in Table~\ref{tab:TR}, TNGA+ and TNLS achieve the best \emph{Eff.} values, implying that they utilize the fewest parameters to achieve the approximation accuracy similar to others.
Most TR methods with rank selection obtain poor performance in this case, and the performance of the two TN-SS methods also becomes worse with increasing the tensor order.
We also see that ``Greedy'' and TNGA$^\star$ cannot guarantee the solution preserving the TR format, even though the data are generated with TR.
Last, we observe that TNLS costs significantly fewer evaluations than TNGA+ to achieve the same \emph{Eff.}.
To see this, we generate another two tensors of order-$8,12$, respectively, and employ TNGA+ and TNLS to search the optimal TR structures.
Figure~\ref{fig:convergence} shows the loss with varying the number of evaluations, averaged from five \textit{i.i.d.} runs.
We see that TNLS obtains consistently faster loss reduction than TNGA+ with the same number of evaluations.

As well as TT/TR, we also verify the effectiveness of our method in other TN formats, including TTree (order-7), PEPS~(order-6), hierarchical Tucker~(HT, order-6,~\citealt{hackbusch2009new}) and multi-scale entanglement renormalization ansatz~(MERA, order-8,~\citealt{cincio2008multiscale,reyes2020multi}), which are widely known in both machine learning and physics.
Being consistent with the results for TT/TR,
the results in Table~\ref{tab:TN} also demonstrate the effectiveness of our method.

	\begin{table*}[t]
	\caption{Average compression ratio$\uparrow$ ($\log$) and RSE$\downarrow$ (in parentheses) for \emph{image compression}, where ``Reshape'' and ``VDT'', denote two different tensorization operations, and the values in square brackets are the number of evaluations required by methods.}
	\vspace{0.2cm}
	\centering
	\begin{threeparttable}
		
		\centering\small  	
		\begin{tabular}{lccccc}
			\toprule

			&\textbf{TR-SVD}& \textbf{TR-LM} & \textbf{Greedy} & \textbf{TNGA+}&\textbf{TNLS}\\
			\midrule
			Reshape&0.92~(0.15)&0.90~(0.14)&0.95~(0.15)&\textbf{1.36}~(\textbf{0.13})~[5700]&1.32~(0.14)~[\textbf{1876}]\\
			VDT&1.10~(0.15)&1.07~(0.15)&0.91~(0.15)&1.28~(0.15)~[5700]&\textbf{1.30}~(0.15)~[\textbf{1546}]\\

			\bottomrule
			
		\end{tabular}\label{tab:compression}
	\end{threeparttable}
	\vskip -0.1in
\end{table*}
\begin{table*}[t]
	\caption{Number of parameters ($\times{}1000$) and mean square error (MSE, in round brackets) for \emph{TGP model compression}, where CCPP, MG and Protein are three datasets, and the values in [square brackets] show the number of evaluations required in each method.}
	\vspace{0.3cm}
	\centering\small
	\begin{threeparttable}\centering	
		\begin{tabular}{lccc}
			\toprule

			& \textbf{CCPP} & \textbf{MG} & \textbf{Protein} \\
			\midrule
			
			\textbf{TGP}~\cite{izmailov2018scalable}&2.64~(0.06)~[N/A]&3.36~(0.33)~[N/A]&2.88~(0.74)~[N/A]\\
			\textbf{TNGA+}&\textbf{2.24}~(0.06)~[1500]&\textbf{3.01}~(0.33)~[4900]&2.03~(0.74)~[3900]\\
			\textbf{TNLS}&\textbf{2.24}~(0.06)~[\textbf{1051}]&\textbf{3.01}~(0.33)~[\textbf{3901}]&\textbf{1.88}~(0.74)~[\textbf{3601}]\\
			\bottomrule
			
		\end{tabular}
	\end{threeparttable}\label{tab:GP}
	\vskip -0.1in
\end{table*}
\subsection{Real-World Data}
We now apply the proposed method to three benchmark problems on real-world data: \emph{image completion},
\emph{image compression},
and \emph{model compression} of tensorial Gaussian process~(TGP).
In these problems, TT/TR has been widely exploited~\cite{wang2017efficient,izmailov2018scalable,yuan2019tensor}.
The goal of the experiment is to verify whether imposing TN-PS can further improve their performance.
Note that in all benchmarks TNGA$+$ and TNLS are implemented in the TR format.
See Appendix for experimental details.

\noindent
\textbf{Image completion.}
TNLS is utilized to predict the missing values from natural images.
In the experiment, seven images from USC-SIPI~\cite{weber1997usc} are chosen, resized by $256\times{}256$, and then reshaped by visual data tensorization~(VDT)~\cite{latorre2005image,bengua2017efficient,yuan2019high} into a tensor of order-9.
After that, the entries are uniformly removed at random with the missing rate $\{70\%,90\%\}$, respectively.
The average of the prediction RSE is shown in Table~\ref{tab:completion}.

\noindent
\textbf{Image compression.}
Tensor decomposition methods are utilized to compress ten natural images randomly chosen from BSD500~\cite{arbelaez2010contour}.
In the data preparation phase, the images are grayscaled, resized by $256\times{}256$, and then reshaped into order-$8$ tensors by both the VDT and the trivial reshaping operations as done in \emph{Python}.
The results are shown in Table~\ref{tab:compression}, where we demonstrate the compression ratio~(in the form of logarithm base $10$) and the corresponding RSE.
For comparison, we implement the methods, including TR-SVD, TR-LM, Greedy, TNGA+, and TNLS, in the experiment.

\noindent
\textbf{Compressing TGP models.}
In this task, we consider compressing not data but parameters of a learning model.
To be specific, we compress the well-trained high-dimensional variational mean of TGP~\cite{izmailov2018scalable} by tensor decomposition.
We evaluate the performance using three datasets for the regression task, including CCPP~\cite{tufekci2014prediction}, MG~\cite{flake2002efficient}, and Protein~\cite{Dua2019}, for which we have the targeted tensors of the order-$\{4,6,9\}$, respectively.
Table~\ref{tab:GP} shows the number of the parameters ($\times{}1000$) after decomposition and the corresponding mean square error (MSE, in the round brackets) for each dataset.

\noindent
\textbf{Results.}
We can observe from the experimental results that TN-PS can boost the performance of the TR models in all tasks.
With the search of vertex permutations, \textit{i.e.}, the ``mode-vertex'' mappings, the expressive and generalization power of the TR models can be significantly improved.
Compared with TN-SS methods like Greedy, TN-PS takes more ``inductive bias'' modeled as the template.
As a consequence, imposing suitable ``inductive bias'' accelerates the searching process and helps avoid the loss in high-dimensional landscapes. 
Compared with TNGA$+$, TNLS achieves similar performance in three tasks and costs significantly less number of evaluations.
This result is expected as we mentioned that the local sampling could leverage the more efficient ``steepest-descent'' path, which is not thought of in the GA-based methods.

\section{Concluding Remarks}
The experiential results demonstrate that TN-PS, a new variant of searching TN structures, can further improve the expressive power of TNs in various tasks.
The new searching algorithm TNLS is verified as being more efficient than existing sampling-based algorithms, with fewer evaluations and faster convergence.

Our theoretical results analyze how the symmetry of TN formats determines the number of all possible ``mode-vertex'' mappings, \textit{i.e.}, the counting property, proving that a universal bound on the counting property exists if the TN formats are sufficiently sparse.
We also establish the basic geometry of the search space for TN-PS.
By the graph isomorphism relation of the TN structures, we construct a semi-metric function and prove its corresponding neighborhood for the search space.
There results are applied as the theoretical foundation to the proposed sampling algorithm.
Taken together, TN-PS explores more efficient TN structures for tensor decomposition/completion/learning tasks, preserving the TN formats in contrast to the previous TN-SS problem.

\noindent\textbf{Limitation.}
One main limitation of our method is the higher running time compared with the greedy method~\cite{hashemizadeh2020adaptive} in searching.
A rigorous analysis about the smoothness of the landscape of TN-SS/PS also remains open.
Our code is available at~\url{https://github.com/ChaoLiAtRIKEN/TNLS}.

\section*{Acknowledgements}
We appreciate the anonymous (meta-)reviewers of ICML~2022 for helpful comments.
We are indebted to Chunmei for the help in proofreading the manuscript, Cichocki and Yokota for the insightful suggestions on the motivation, and Pan and Hashemizadeh for comments that improved the manuscript. 
This work was partially supported by JSPS KAKENHI (Grant No. 20K19875, 20H04249, 20H04208) and the National Natural Science Foundation of China (Grant No. 62006045, 62071132).
Part of the computation was carried out at the Riken AIp Deep learning ENvironment~(RAIDEN).

\bibliography{references.bib}

\begin{thebibliography}{71}
\providecommand{\natexlab}[1]{#1}
\providecommand{\url}[1]{\texttt{#1}}
\expandafter\ifx\csname urlstyle\endcsname\relax
  \providecommand{\doi}[1]{doi: #1}\else
  \providecommand{\doi}{doi: \begingroup \urlstyle{rm}\Url}\fi

\bibitem[Arbelaez et~al.(2010)Arbelaez, Maire, Fowlkes, and
  Malik]{arbelaez2010contour}
Arbelaez, P., Maire, M., Fowlkes, C., and Malik, J.
\newblock Contour detection and hierarchical image segmentation.
\newblock \emph{IEEE transactions on pattern analysis and machine
  intelligence}, 33\penalty0 (5):\penalty0 898--916, 2010.

\bibitem[Bachmayr et~al.(2016)Bachmayr, Schneider, and
  Uschmajew]{bachmayr2016tensor}
Bachmayr, M., Schneider, R., and Uschmajew, A.
\newblock Tensor networks and hierarchical tensors for the solution of
  high-dimensional partial differential equations.
\newblock \emph{Foundations of Computational Mathematics}, 16\penalty0
  (6):\penalty0 1423--1472, 2016.

\bibitem[Back(1996)]{back1996evolutionary}
Back, T.
\newblock \emph{Evolutionary algorithms in theory and practice: evolution
  strategies, evolutionary programming, genetic algorithms}.
\newblock Oxford university press, 1996.

\bibitem[Batselier(2018)]{batselier2018trouble}
Batselier, K.
\newblock The trouble with tensor ring decompositions.
\newblock \emph{arXiv preprint arXiv:1811.03813}, 2018.

\bibitem[Bean(1994)]{bean1994genetic}
Bean, J.~C.
\newblock Genetic algorithms and random keys for sequencing and optimization.
\newblock \emph{ORSA journal on computing}, 6\penalty0 (2):\penalty0 154--160,
  1994.

\bibitem[Bengua et~al.(2017)Bengua, Phien, Tuan, and Do]{bengua2017efficient}
Bengua, J.~A., Phien, H.~N., Tuan, H.~D., and Do, M.~N.
\newblock Efficient tensor completion for color image and video recovery:
  Low-rank tensor train.
\newblock \emph{IEEE Transactions on Image Processing}, 26\penalty0
  (5):\penalty0 2466--2479, 2017.

\bibitem[Cai \& Li(2021)Cai and Li]{cai2021blind}
Cai, Y. and Li, P.
\newblock A blind block term decomposition of high order tensors.
\newblock In \emph{Proceedings of the AAAI Conference on Artificial
  Intelligence}, volume~35, pp.\  6868--6876, 2021.

\bibitem[Chang et~al.(1995)Chang, Gasarch, and Toran]{chang1995finding}
Chang, R., Gasarch, W., and Toran, J.
\newblock On finding the number of graph automorphisms.
\newblock In \emph{Proceedings of Structure in Complexity Theory. Tenth Annual
  IEEE Conference}, pp.\  288--298. IEEE, 1995.

\bibitem[Cheng et~al.(2020)Cheng, Li, Fan, and Bao]{cheng2020novel}
Cheng, Z., Li, B., Fan, Y., and Bao, Y.
\newblock A novel rank selection scheme in tensor ring decomposition based on
  reinforcement learning for deep neural networks.
\newblock In \emph{ICASSP 2020-2020 IEEE International Conference on Acoustics,
  Speech and Signal Processing (ICASSP)}, pp.\  3292--3296. IEEE, 2020.

\bibitem[Cichocki et~al.(2016)Cichocki, Lee, Oseledets, Phan, Zhao, Mandic,
  et~al.]{cichocki2016tensor}
Cichocki, A., Lee, N., Oseledets, I., Phan, A.-H., Zhao, Q., Mandic, D.~P.,
  et~al.
\newblock Tensor networks for dimensionality reduction and large-scale
  optimization: Part 1 low-rank tensor decompositions.
\newblock \emph{Foundations and Trends{\textregistered} in Machine Learning},
  9\penalty0 (4-5):\penalty0 249--429, 2016.

\bibitem[Cincio et~al.(2008)Cincio, Dziarmaga, and Rams]{cincio2008multiscale}
Cincio, L., Dziarmaga, J., and Rams, M.~M.
\newblock Multiscale entanglement renormalization ansatz in two dimensions:
  quantum ising model.
\newblock \emph{Physical Review Letters}, 100\penalty0 (24):\penalty0 240603,
  2008.

\bibitem[Cohen et~al.(2016)Cohen, Sharir, and Shashua]{cohen2016expressive}
Cohen, N., Sharir, O., and Shashua, A.
\newblock On the expressive power of deep learning: A tensor analysis.
\newblock In \emph{Conference on learning theory}, pp.\  698--728. PMLR, 2016.

\bibitem[Dua \& Graff(2017)Dua and Graff]{Dua2019}
Dua, D. and Graff, C.
\newblock {UCI} machine learning repository, 2017.
\newblock URL \url{http://archive.ics.uci.edu/ml}.

\bibitem[Falc{\'o} et~al.(2020)Falc{\'o}, Nouy, et~al.]{falco2020geometry}
Falc{\'o}, A., Nouy, W.~H., et~al.
\newblock Geometry of tree-based tensor formats in tensor banach spaces.
\newblock \emph{arXiv preprint arXiv:2011.08466}, 2020.

\bibitem[Flake \& Lawrence(2002)Flake and Lawrence]{flake2002efficient}
Flake, G.~W. and Lawrence, S.
\newblock Efficient {SVM} regression training with {SMO}.
\newblock \emph{Machine Learning}, 46\penalty0 (1):\penalty0 271--290, 2002.

\bibitem[Glasser et~al.(2019)Glasser, Sweke, Pancotti, Eisert, and
  Cirac]{glasser2019expressive}
Glasser, I., Sweke, R., Pancotti, N., Eisert, J., and Cirac, I.
\newblock Expressive power of tensor-network factorizations for probabilistic
  modeling.
\newblock \emph{Advances in neural information processing systems}, 32, 2019.

\bibitem[Goldwasser et~al.(1989)Goldwasser, Micali, and
  Rackoff]{goldwasser1989knowledge}
Goldwasser, S., Micali, S., and Rackoff, C.
\newblock The knowledge complexity of interactive proof systems.
\newblock \emph{SIAM Journal on computing}, 18\penalty0 (1):\penalty0 186--208,
  1989.

\bibitem[Golovin et~al.(2019)Golovin, Karro, Kochanski, Lee, Song, and
  Zhang]{golovin2019gradientless}
Golovin, D., Karro, J., Kochanski, G., Lee, C., Song, X., and Zhang, Q.
\newblock Gradientless descent: High-dimensional zeroth-order optimization.
\newblock In \emph{International Conference on Learning Representations}, 2019.

\bibitem[Haberstich et~al.(2021)Haberstich, Nouy, and
  Perrin]{haberstich2021active}
Haberstich, C., Nouy, A., and Perrin, G.
\newblock Active learning of tree tensor networks using optimal least-squares.
\newblock \emph{arXiv preprint arXiv:2104.13436}, 2021.

\bibitem[Hackbusch \& K{\"u}hn(2009)Hackbusch and K{\"u}hn]{hackbusch2009new}
Hackbusch, W. and K{\"u}hn, S.
\newblock A new scheme for the tensor representation.
\newblock \emph{Journal of Fourier analysis and applications}, 15\penalty0
  (5):\penalty0 706--722, 2009.

\bibitem[Hashemizadeh et~al.(2020)Hashemizadeh, Liu, Miller, and
  Rabusseau]{hashemizadeh2020adaptive}
Hashemizadeh, M., Liu, M., Miller, J., and Rabusseau, G.
\newblock Adaptive tensor learning with tensor networks.
\newblock \emph{arXiv preprint arXiv:2008.05437}, 2020.

\bibitem[Hawkins \& Zhang(2021)Hawkins and Zhang]{hawkins2021bayesian}
Hawkins, C. and Zhang, Z.
\newblock Bayesian tensorized neural networks with automatic rank selection.
\newblock \emph{Neurocomputing}, 453:\penalty0 172--180, 2021.

\bibitem[Hayashi et~al.(2019)Hayashi, Yamaguchi, Sugawara, and
  Maeda]{hayashi2019exploring}
Hayashi, K., Yamaguchi, T., Sugawara, Y., and Maeda, S.-i.
\newblock Exploring unexplored tensor network decompositions for convolutional
  neural networks.
\newblock In \emph{Advances in Neural Information Processing Systems}, pp.\
  5553--5563, 2019.

\bibitem[Huggins et~al.(2019)Huggins, Patil, Mitchell, Whaley, and
  Stoudenmire]{huggins2019towards}
Huggins, W., Patil, P., Mitchell, B., Whaley, K.~B., and Stoudenmire, E.~M.
\newblock Towards quantum machine learning with tensor networks.
\newblock \emph{Quantum Science and technology}, 4\penalty0 (2):\penalty0
  024001, 2019.

\bibitem[Izmailov et~al.(2018)Izmailov, Novikov, and
  Kropotov]{izmailov2018scalable}
Izmailov, P., Novikov, A., and Kropotov, D.
\newblock Scalable {G}aussian processes with billions of inducing inputs via
  tensor train decomposition.
\newblock In \emph{International Conference on Artificial Intelligence and
  Statistics}, pp.\  726--735. PMLR, 2018.

\bibitem[Khavari \& Rabusseau(2021)Khavari and Rabusseau]{khavari2021lower}
Khavari, B. and Rabusseau, G.
\newblock Lower and upper bounds on the pseudo-dimension of tensor network
  models.
\newblock \emph{Advances in Neural Information Processing Systems}, 34, 2021.

\bibitem[Kingma \& Ba(2014)Kingma and Ba]{kingma2014adam}
Kingma, D.~P. and Ba, J.
\newblock Adam: A method for stochastic optimization.
\newblock \emph{arXiv preprint arXiv:1412.6980}, 2014.

\bibitem[Kodryan et~al.(2020)Kodryan, Kropotov, and Vetrov]{kodryan2020mars}
Kodryan, M., Kropotov, D., and Vetrov, D.
\newblock Mars: Masked automatic ranks selection in tensor decompositions.
\newblock \emph{arXiv preprint arXiv:2006.10859}, 2020.

\bibitem[Kossaifi et~al.(2020)Kossaifi, Lipton, Kolbeinsson, Khanna,
  Furlanello, and Anandkumar]{kossaifi2020tensor}
Kossaifi, J., Lipton, Z.~C., Kolbeinsson, A., Khanna, A., Furlanello, T., and
  Anandkumar, A.
\newblock Tensor regression networks.
\newblock \emph{Journal of Machine Learning Research}, 21:\penalty0 1--21,
  2020.

\bibitem[Koutra et~al.(2011)Koutra, Parikh, Ramdas, and
  Xiang]{koutra2011algorithms}
Koutra, D., Parikh, A., Ramdas, A., and Xiang, J.
\newblock Algorithms for graph similarity and subgraph matching.
\newblock In \emph{Proc. Ecol. Inference Conf}, volume~17, 2011.

\bibitem[Krasikov et~al.(2006)Krasikov, Lev, and Thatte]{krasikov2006upper}
Krasikov, I., Lev, A., and Thatte, B.
\newblock Upper bounds on the automorphism group of a graph.
\newblock \emph{Discrete Math.}, 256\penalty0 (math. CO/0609425):\penalty0
  489--493, 2006.

\bibitem[Landsberg et~al.(2011)Landsberg, Qi, and Ye]{landsberg2011geometry}
Landsberg, J.~M., Qi, Y., and Ye, K.
\newblock On the geometry of tensor network states.
\newblock \emph{arXiv preprint arXiv:1105.4449}, 2011.

\bibitem[Latorre(2005)]{latorre2005image}
Latorre, J.~I.
\newblock Image compression and entanglement.
\newblock \emph{arXiv preprint quant-ph/0510031}, 2005.

\bibitem[Li \& Sun(2020)Li and Sun]{li2020evolutionary}
Li, C. and Sun, Z.
\newblock Evolutionary topology search for tensor network decomposition.
\newblock In \emph{Proceedings of the 37th International Conference on Machine
  Learning (ICML)}, 2020.

\bibitem[Li et~al.(2020)Li, Sun, and Zhao]{li2020high}
Li, C., Sun, Z., and Zhao, Q.
\newblock High-order learning model via fractional tensor network
  decomposition.
\newblock In \emph{First Workshop on Quantum Tensor Networks in Machine
  Learning, 34th Conference on Neural Information Processing Systems (NeurIPS
  2020)}, 2020.

\bibitem[Li et~al.(2021)Li, Pan, Chen, Ding, Zhao, and Xu]{li2021heuristic}
Li, N., Pan, Y., Chen, Y., Ding, Z., Zhao, D., and Xu, Z.
\newblock Heuristic rank selection with progressively searching tensor ring
  network.
\newblock \emph{Complex \& Intelligent Systems}, pp.\  1--15, 2021.

\bibitem[Liu et~al.(2021)Liu, Li, Zhang, and Zhang]{liu2021tensor}
Liu, J., Li, S., Zhang, J., and Zhang, P.
\newblock Tensor networks for unsupervised machine learning.
\newblock \emph{arXiv preprint arXiv:2106.12974}, 2021.

\bibitem[Liu et~al.(2020{\natexlab{a}})Liu, Chen, Kailkhura, Zhang, Hero~III,
  and Varshney]{liu2020primer}
Liu, S., Chen, P.-Y., Kailkhura, B., Zhang, G., Hero~III, A.~O., and Varshney,
  P.~K.
\newblock A primer on zeroth-order optimization in signal processing and
  machine learning: Principals, recent advances, and applications.
\newblock \emph{IEEE Signal Processing Magazine}, 37\penalty0 (5):\penalty0
  43--54, 2020{\natexlab{a}}.

\bibitem[Liu et~al.(2020{\natexlab{b}})Liu, Lu, Chen, Feng, Xu, Al-Dujaili,
  Hong, and O’Reilly]{liu2020min}
Liu, S., Lu, S., Chen, X., Feng, Y., Xu, K., Al-Dujaili, A., Hong, M., and
  O’Reilly, U.-M.
\newblock Min-max optimization without gradients: Convergence and applications
  to black-box evasion and poisoning attacks.
\newblock In \emph{International Conference on Machine Learning}, pp.\
  6282--6293. PMLR, 2020{\natexlab{b}}.

\bibitem[Long et~al.(2021)Long, Zhu, Liu, and Liu]{long2021bayesian}
Long, Z., Zhu, C., Liu, J., and Liu, Y.
\newblock Bayesian low rank tensor ring for image recovery.
\newblock \emph{IEEE Transactions on Image Processing}, 30:\penalty0
  3568--3580, 2021.

\bibitem[L{\"u}ck(2008)]{luck2008survey}
L{\"u}ck, W.
\newblock Survey on geometric group theory.
\newblock \emph{arXiv preprint arXiv:0806.3771}, 2008.

\bibitem[Markov \& Shi(2008)Markov and Shi]{markov2008simulating}
Markov, I.~L. and Shi, Y.
\newblock Simulating quantum computation by contracting tensor networks.
\newblock \emph{SIAM Journal on Computing}, 38\penalty0 (3):\penalty0 963--981,
  2008.

\bibitem[Mickelin \& Karaman(2020)Mickelin and Karaman]{mickelin2020algorithms}
Mickelin, O. and Karaman, S.
\newblock On algorithms for and computing with the tensor ring decomposition.
\newblock \emph{Numerical Linear Algebra with Applications}, 27\penalty0
  (3):\penalty0 e2289, 2020.

\bibitem[Miller et~al.(2021)Miller, Rabusseau, and Terilla]{miller2021tensor}
Miller, J., Rabusseau, G., and Terilla, J.
\newblock Tensor networks for probabilistic sequence modeling.
\newblock In \emph{International Conference on Artificial Intelligence and
  Statistics}, pp.\  3079--3087. PMLR, 2021.

\bibitem[Nie et~al.(2021)Nie, Wang, and Tian]{nie2021adaptive}
Nie, C., Wang, H., and Tian, L.
\newblock Adaptive tensor networks decomposition.
\newblock In \emph{BMVC}, 2021.

\bibitem[Novikov et~al.(2015)Novikov, Podoprikhin, Osokin, and
  Vetrov]{novikov2015tensorizing}
Novikov, A., Podoprikhin, D., Osokin, A., and Vetrov, D.~P.
\newblock Tensorizing neural networks.
\newblock In \emph{Advances in Neural Information Processing Systems}, pp.\
  442--450, 2015.

\bibitem[Novikov et~al.(2021)Novikov, Panov, and Oseledets]{novikov2021tensor}
Novikov, G.~S., Panov, M.~E., and Oseledets, I.~V.
\newblock Tensor-train density estimation.
\newblock In \emph{Uncertainty in Artificial Intelligence}, pp.\  1321--1331.
  PMLR, 2021.

\bibitem[Oseledets(2011)]{oseledets2011tensor}
Oseledets, I.~V.
\newblock Tensor-train decomposition.
\newblock \emph{SIAM Journal on Scientific Computing}, 33\penalty0
  (5):\penalty0 2295--2317, 2011.

\bibitem[Papadimitriou \& Yannakakis(1982)Papadimitriou and
  Yannakakis]{papadimitriou1982complexity}
Papadimitriou, C.~H. and Yannakakis, M.
\newblock The complexity of facets (and some facets of complexity).
\newblock In \emph{Proceedings of the fourteenth annual ACM symposium on Theory
  of computing}, pp.\  255--260, 1982.

\bibitem[Qiu et~al.(2021)Qiu, Li, Weng, Sun, He, and Zhao]{qiu2021memory}
Qiu, H., Li, C., Weng, Y., Sun, Z., He, X., and Zhao, Q.
\newblock On the memory mechanism of tensor-power recurrent models.
\newblock In \emph{International Conference on Artificial Intelligence and
  Statistics}, pp.\  3682--3690. PMLR, 2021.

\bibitem[Razin et~al.(2021)Razin, Maman, and Cohen]{razin2021implicit}
Razin, N., Maman, A., and Cohen, N.
\newblock Implicit regularization in tensor factorization.
\newblock In \emph{International Conference on Machine Learning}, pp.\
  8913--8924. PMLR, 2021.

\bibitem[Razin et~al.(2022)Razin, Maman, and Cohen]{razin2022implicit}
Razin, N., Maman, A., and Cohen, N.
\newblock Implicit regularization in hierarchical tensor factorization and deep
  convolutional neural networks.
\newblock \emph{arXiv preprint arXiv:2201.11729}, 2022.

\bibitem[Reyes \& Stoudenmire(2020)Reyes and Stoudenmire]{reyes2020multi}
Reyes, J. and Stoudenmire, M.
\newblock A multi-scale tensor network architecture for classification and
  regression.
\newblock \emph{arXiv preprint arXiv:2001.08286}, 2020.

\bibitem[Richter et~al.(2021)Richter, Sallandt, and
  N{\"u}sken]{pmlr-v139-richter21a}
Richter, L., Sallandt, L., and N{\"u}sken, N.
\newblock Solving high-dimensional parabolic {PDE}s using the tensor train
  format.
\newblock In Meila, M. and Zhang, T. (eds.), \emph{Proceedings of the 38th
  International Conference on Machine Learning}, volume 139 of
  \emph{Proceedings of Machine Learning Research}, pp.\  8998--9009. PMLR,
  18--24 Jul 2021.
\newblock URL \url{https://proceedings.mlr.press/v139/richter21a.html}.

\bibitem[Russell \& Norvig(1995)Russell and Norvig]{russell1995artificial}
Russell, S. and Norvig, P.
\newblock \emph{Artificial Intelligence: A Modern Approach}.
\newblock CUMINCAD, 1995.

\bibitem[Savarese et~al.(2021)Savarese, McAllester, Babu, and
  Maire]{savarese2021domain}
Savarese, P., McAllester, D., Babu, S., and Maire, M.
\newblock Domain-independent dominance of adaptive methods.
\newblock In \emph{Proceedings of the IEEE/CVF Conference on Computer Vision
  and Pattern Recognition}, pp.\  16286--16295, 2021.

\bibitem[Sedighin et~al.(2021)Sedighin, Cichocki, and
  Phan]{sedighin2021adaptive}
Sedighin, F., Cichocki, A., and Phan, A.-H.
\newblock Adaptive rank selection for tensor ring decomposition.
\newblock \emph{IEEE Journal of Selected Topics in Signal Processing},
  15\penalty0 (3):\penalty0 454--463, 2021.

\bibitem[Singh(2021)]{singh2021continuum}
Singh, S.
\newblock Continuum-armed bandits: A function space perspective.
\newblock In \emph{International Conference on Artificial Intelligence and
  Statistics}, pp.\  2620--2628. PMLR, 2021.

\bibitem[Snyder \& Daskin(2006)Snyder and Daskin]{snyder2006random}
Snyder, L.~V. and Daskin, M.~S.
\newblock A random-key genetic algorithm for the generalized traveling salesman
  problem.
\newblock \emph{European journal of operational research}, 174\penalty0
  (1):\penalty0 38--53, 2006.

\bibitem[Tao \& Zhao(2020)Tao and Zhao]{tao2020bayesian}
Tao, Z. and Zhao, Q.
\newblock Bayesian tensor ring decomposition for low rank tensor completion.
\newblock In \emph{International Workshop on Tensor Network Representations in
  Machine Learning, IJCAI}, 2020.

\bibitem[T{\"u}fekci(2014)]{tufekci2014prediction}
T{\"u}fekci, P.
\newblock Prediction of full load electrical power output of a base load
  operated combined cycle power plant using machine learning methods.
\newblock \emph{International Journal of Electrical Power \& Energy Systems},
  60:\penalty0 126--140, 2014.

\bibitem[Verstraete \& Cirac(2004)Verstraete and
  Cirac]{verstraete2004renormalization}
Verstraete, F. and Cirac, J.~I.
\newblock Renormalization algorithms for quantum-many body systems in two and
  higher dimensions.
\newblock \emph{arXiv preprint cond-mat/0407066}, 2004.

\bibitem[Wang et~al.(2017)Wang, Aggarwal, and Aeron]{wang2017efficient}
Wang, W., Aggarwal, V., and Aeron, S.
\newblock Efficient low rank tensor ring completion.
\newblock In \emph{Proceedings of the IEEE International Conference on Computer
  Vision}, pp.\  5697--5705, 2017.

\bibitem[Weber(1997)]{weber1997usc}
Weber, A.~G.
\newblock The {USC-SIPI} image database version 5.
\newblock \emph{USC-SIPI Report}, 315\penalty0 (1), 1997.

\bibitem[Ye \& Lim(2019)Ye and Lim]{Ye2019Tensor}
Ye, K. and Lim, L.-H.
\newblock Tensor network ranks.
\newblock \emph{arXiv preprint arXiv:1801.02662}, 2019.

\bibitem[Yokota et~al.(2016)Yokota, Zhao, and Cichocki]{yokota2016smooth}
Yokota, T., Zhao, Q., and Cichocki, A.
\newblock Smooth parafac decomposition for tensor completion.
\newblock \emph{IEEE Transactions on Signal Processing}, 64\penalty0
  (20):\penalty0 5423--5436, 2016.

\bibitem[Yuan et~al.(2019{\natexlab{a}})Yuan, Li, Mandic, Cao, and
  Zhao]{yuan2019tensor}
Yuan, L., Li, C., Mandic, D., Cao, J., and Zhao, Q.
\newblock Tensor ring decomposition with rank minimization on latent space: An
  efficient approach for tensor completion.
\newblock In \emph{Proceedings of the AAAI Conference on Artificial
  Intelligence}, volume~33, pp.\  9151--9158, 2019{\natexlab{a}}.

\bibitem[Yuan et~al.(2019{\natexlab{b}})Yuan, Zhao, Gui, and Cao]{yuan2019high}
Yuan, L., Zhao, Q., Gui, L., and Cao, J.
\newblock High-order tensor completion via gradient-based optimization under
  tensor train format.
\newblock \emph{Signal Processing: Image Communication}, 73:\penalty0 53--61,
  2019{\natexlab{b}}.

\bibitem[Zhao et~al.(2015)Zhao, Zhang, and Cichocki]{zhao2015bayesian}
Zhao, Q., Zhang, L., and Cichocki, A.
\newblock Bayesian {CP} factorization of incomplete tensors with automatic rank
  determination.
\newblock \emph{IEEE transactions on pattern analysis and machine
  intelligence}, 37\penalty0 (9):\penalty0 1751--1763, 2015.

\bibitem[Zhao et~al.(2016)Zhao, Zhou, Xie, Zhang, and Cichocki]{zhao2016tensor}
Zhao, Q., Zhou, G., Xie, S., Zhang, L., and Cichocki, A.
\newblock Tensor ring decomposition.
\newblock \emph{arXiv preprint arXiv:1606.05535}, 2016.

\bibitem[Zheng et~al.(2021)Zheng, Huang, Zhao, Zhao, and Jiang]{zheng2021fully}
Zheng, Y.-B., Huang, T.-Z., Zhao, X.-L., Zhao, Q., and Jiang, T.-X.
\newblock Fully-connected tensor network decomposition and its application to
  higher-order tensor completion.
\newblock In \emph{Proc. AAAI}, volume~35, pp.\  11071--11078, 2021.

\end{thebibliography}
\bibliographystyle{icml2022}

\newpage
\appendix
\onecolumn
In the appendix, we first give the proofs for the results mentioned in the main body of the paper.
After that, more details of the experiments, including tuning parameter settings and additional experimental results, will be introduced.

\section{Proofs}\label{sec:Proofs}
\subsection{Proof of Lemma~\ref{thm:CountingLemma}}
\begin{proof}
	To obtain the result, we first have the following inequalities:
	\begin{equation}
	\begin{split}
	&\log\vert{}Aut(G_0)\vert-\log\vert{}\mathbb{S}_N\vert\leq{}\log{}N+\log\Delta!+(N-\Delta-1)\log(\Delta-1)-\log{}N!\\
	&\leq{}
	(1/2-N)\log{}N+(\Delta+1/2)\log{}\Delta+(N-\Delta-1)\log(\Delta-1)+N-\Delta-\frac{1}{12N+1}+\frac{1}{12\Delta}\\
	&\leq{}
	(1/2-N)\log\frac{N}{\Delta}+N-\Delta+1
	=(1/2-N)\log{}d+N-N/d+1/12,
	\end{split}\label{eq:L32Ineq}
	\end{equation}
	where the first inequality follows from Theorem 2 given in~\cite{krasikov2006upper}.
	In this theorem, it is proved that $\vert{}Aut(G_0)\vert$ is above bounded by the maximum graph degree, written $\Delta$, as follows:
	\begin{equation}
	\log\vert{}Aut(G_0)\vert\leq{}\log{}N+\log\Delta!+(N-\Delta-1)\log(\Delta-1).
	\end{equation}
	The second inequality of~\eqref{eq:L32Ineq} follows by $\vert\mathbb{S}_N\vert=N!$ and Stirling approximation of factorials, by which the terms $\log\Delta!$ and $-\log{}N!$ are bounded as follows:
	\begin{equation}
	\begin{split}
	&\log\Delta!\leq{}0.5\log{}2\pi+(\Delta+1/2)\log{\Delta}-\Delta+\frac{1}{12\Delta},
	\end{split}
	\end{equation}
	and
	\begin{equation}
	\begin{split}
	&-\log{}N!\leq{}-0.5\log{}2\pi-(N+1/2)\log{}N+N-\frac{1}{12N+1},
	\end{split}
	\end{equation}
	respectively.
	In the third line of~\eqref{eq:L32Ineq}, the (in-)equalities follows from: $\log(\Delta-1)\leq{}\log(\Delta)$, $1/\left(12\Delta\right)\leq{}1/24$ and $N>0$, and the assumption $N/\Delta=d$.
	The proof is thus accomplished by eliminating the logarithm on the both sides of the inequality.
\end{proof}

\subsection{Proof of Theorem~\ref{thm:OverallBounds}}
\begin{proof}
	According to the Lagrange's theorem in group theory, the size of $\mathbb{L}_{G_0,R}$ is equal to 
	\begin{equation}
	\vert{}{}\mathbb{L}_{G_0,R}\vert=\frac{\vert{}\mathbb{S}_N\vert\cdot\vert{}\mathbb{F}_{G_0,R}\vert}{\vert{}Aut(G_0)\vert}=\frac{\vert{}\mathbb{S}_N\vert\cdot\vert{}\mathbb{Z}_R\vert^{\vert\mathbb{E}_0\vert}}{\vert{}Aut(G_0)\vert},\label{eq:Size1}
	\end{equation}
	where the equation $\vert\mathbb{F}_{G_0,R}\vert=\vert{}\mathbb{Z}_R\vert^{\vert\mathbb{E}_0\vert}$ holds by the TN-PS model.
	By the handshaking lemma in graph theory,
	\begin{equation}
	\vert{}\mathbb{E}_0\vert=\frac{1}{2}\sum_{n=1}^N{}deg(v_n),
	\end{equation}
	where $G_0=(\mathbb{V},\mathbb{E}_0)$, $v_n\in{}\mathbb{V}$ for $n\in{}[N]$, and $deg(v_n)$ denotes the degree of $v_n$.
	The number of edges is thus bounded by
	\begin{equation}
	\frac{N}{2}\delta\leq{}\vert{}\mathbb{E}_0\vert\leq{}\frac{N}{2}\Delta.\label{eq:Size2}
	\end{equation}
	The inequalities~\eqref{eq:SizeBound} in Theorem~\ref{thm:OverallBounds} are finally obtained by combing Lemma~\ref{thm:CountingLemma} to \eqref{eq:Size1} and \eqref{eq:Size2}.
\end{proof}

\subsection{Proof of Lemma~\ref{thm:metric}}
\begin{proof}
	First, we prove that $(\mathbb{G}_0,d_{G_0})$ defines a semi-metric space.
	To do this, we should prove the function $d_{G_0}$ defined by~\eqref{eq:metric} satisfying the following claims:
	\begin{itemize}
		\item[(a)] $d_{G_0}(G_1,G_2)>0$ if $G_1\neq{}G_2$; $d_{G_0}(G, G)=0$, otherwise;
		\item[(b)] $d_{G_0}(G_1,G_2)=d_{G_0}(G_2,G_1)$;
	\end{itemize}
	for $G_1,G_2,G\in{}\mathbb{G}_0$.
	We first see that the three claims are naturally true for a trivial $\mathbb{G}_0$.
	Then in the following we only consider the case of non-trivial $\mathbb{G}_0$.
	To prove the claim (a), we suppose $G_1\neq{}G_2$ with $G_1=g_1\cdot{}G_0$, $G_2=g_2\cdot{}G_0$.
	It thus give $g_1\neq{}g_2$ holds.
	$g_1\cdot{}Aut(G_0)\cap{}g_2\cdot{}Aut(G_0)=\emptyset{}$, since $g_i{}\cdot{}Aut(G_0),\,i=1,2$ are left cosets of $Aut(G_0)$ which partitions $\mathbb{S}_N$.
	We therefore have $p_1\neq{}p_2$ and $d_\mathbb{T}(p_1,p_2)>0$, $\forall{}p_i\in{}g_i\cdot{}Aut(G_0),i=1,2$, by which $d_{G_0}\left(G_1,G_2\right)>0$.
	Suppose conversely that $G_1=G_2$, then we have $g_1^{-1}g_2\in{}Aut(G_0)$.
	We thus know that there exist $\hat{p}\in{}g_1\cdot{}Aut(G_0)$ such that $\hat{p}=g_1g_1^{-1}g_2=g_2$.
	Therefore, it obeys
	\begin{equation}
	\begin{split}
	d_{G_0}\left(G_1,G_2\right)=\min_{
		p_i\in{}g_i\cdot{}Aut(G_0),i=1,2
	}d_{\mathbb{T}_N }(p_1,p_2)\leq{}d_{\mathbb{T}_N }(\hat{p},g_2)=d_{\mathbb{T}_N }(g_2,g_2)=0.
	\end{split}
	\end{equation}
	By $d_{G_0}(G_1,G_2)\geq{}0$, the claim (a) is thus proved.
	
	The claim (b) is obviously true.

	Next, we show the set $\mathbb{N}_N(G)$ defines the neighborhood of $G$ in $\mathbb{G}_0$ associated with $d_{G_0}$.
	We first see that it is trivially true if $D=0$.
	For $D>0$, we prove that $d_{G_0}\left(G',G\right)\leq{}D$ holds for all $G'\in\mathbb{N}_D(G)$.
	By the assumption $G=g\cdot{}G_0$ we first have that  $p\cdot{}G_0=gA\cdot{}G_0=g\cdot{}G_0=G$ holds for all $p\in{}g\cdot{}Aut(G_0)$ where $A\in{}Aut(G_0)$.
	By $G'\in\mathbb{N}_D(G)$, there exists $d_0\in[D]$ such that $G'\in{}\mathbb{I}_{d_0}(G)$.
	Thus
	\begin{equation}
	\begin{split}
	d_{G_0}(G',G)&=\min_{p'\in{}g\cdot{}Aut(G_0)\prod_{i=1}^{d_0}t_i\cdot{}Aut(G_0);\,p\in{}g\cdot{}Aut(G_0)}d_{\mathbb{T}_N}(p',p)\\
	&\leq{}d_{\mathbb{T}_N}\left(gA_1\prod_{i=1}^{d_0}t_iA_2,\,gA_3\right)\\
	&=d_{\mathbb{T}_N}\left(1,\left(gA_1\prod_{i=1}^{d_0}t_iA_2\right)^{-1}gA_3\right)\\
	&=d_{\mathbb{T}_N}\left(1,A_2^{-1}\left(\prod_{i=1}^{d_0}t_i\right)^{-1}A_1^{-1}g^{-1}gA_3\right)=d_{\mathbb{T}_N}\left(1,\prod_{i=1}^{d_0}t_{d_0-i}^{-1}\right)\\
	&=d_{\mathbb{T}_N}\left(1,\prod_{i=1}^{d_0}t_{d_0-i}\right)\leq{}d_0\leq{}D,
	\end{split}\label{eq:dGG}
	\end{equation}
	where $A_i\in{}Aut(G_0),\,i=1,2,3$, $A_1=A_3$ and $A_2$ is equal to the identity of $Aut(G_0)$.
	In~\eqref{eq:dGG}, the fist line follows from the definition of $\mathbb{I}_{d_0}(G)$; the third line holds the left-isometry property of the world metric; the last line holds by the definition of the word metric and the fact $t_i^{-1}=t_i\in{}\mathbb{T}_N$.
	Next, we prove the converse side, that is, $G_x\in{}\mathbb{N}_D(G)$ for all $G_x\in{}\left\{G'\in\mathbb{G}_0\vert{}d_{G_0}(G',G)\leq{}D\right\}$.
	By the definition of $d_{G_0}$,
	\begin{equation}
	\begin{split}
	d_{G_0}(G_x,G)=\min_{p_x\in{}g_x\cdot{}Aut(G_0);p\in{}g\cdot{}Aut(G_0)}d_{\mathbb{T}_N}(p_x,p)\leq{}D
	\end{split}.
	\end{equation}
	Thus, there exist $A_x,A\in{}Aut(G_0)$ such that the inequality
	\begin{equation}
	d_\mathbb{T}(g_xA_x,pA)=d_{\mathbb{T}_N}(1,A_x^{-1}g_x^{-1}gA)\leq{}D
	\end{equation}
	holds.
	Let $g_x=g\cdot{}h_x$ for some $h_x\in\mathbb{S}_N$, then $d_{\mathbb{T}_N}(1,A_x^{-1}h_x^{-1}A)\leq{}D$.
	According to the definition of the word metric $d_{\mathbb{T}_N}$, we know there exists $d\leq{}D$ and a sequence of permutations $\{t_1,t_2,\ldots,t_d\}\subseteq{}\mathbb{T}_N$ such that $A_x^{-1}h_x^{-1}A=\prod_{i=1}^dt_i^{\epsilon}$, where $\epsilon=\{+1,-1\}$~\citep{luck2008survey}.
	Since $t_i^{-1}=t_i$, the equation $h_x=A\prod_{i=1}^dt_i{}A_x^{-1}$ holds, and $G_x=g_xG_0=gh_xG_0=gA\prod_{i=1}^dt_i{}A_x^{-1}G_0=gA\prod_{i=1}^dt_i{}G_0$ consequently holds, where the last equation follows from $A_x^{-1}\in{}Aut(G_0)$.
	The equations say that $G_x$ is an element of $\mathbb{I}_d(G)$, namely, $G_x\in\mathbb{I}_d(G)\subseteq{}N_D(G)$.
	Combing the results from the both two sides, the proof of the lemma is accomplished.
\end{proof}

\subsection{Proof of Theorem~\ref{thm:cover}}
Before the proof of Theorem 3.4, we first restate a classic claim in group theory about the \emph{cycle decomposition} of a permutation under conjugation.
\begin{lemma}\label{thm:conjugate}
	Let $\sigma$ and $\tau$ be two elements of $\mathbb{S}_N$.
	Suppose that $\sigma=(a_1,a_2,\ldots,a_k)(b_1,a_2,\ldots,b_l)\ldots$ is the cycle decomposition of $\sigma$.
	The $\sigma=(\tau(a_1),\tau(a_2),\ldots,\tau(a_k))(\tau(b_1),\tau(a_2),\ldots,\tau(b_l))\ldots$ is the cycle decomposition of $\tau\sigma\tau^{-1}$, the conjugate of $\sigma$ by $\tau$.
\end{lemma}
We then apply this lemma to the following result:
\begin{lemma}\label{thm:cycle}
	Let $\mathbb{C}_N=\left\{(i,j)\vert{}1\leq{}i<j\leq{}N\right\}$ be the collection of all 2-cycles of $\mathbb{S}_N$.g
	For any $G^\prime\in{}\mathbb{I}_d(G)$ where $G=g\cdot{}G_0\in\mathbb{G}_0$, there exist a series of 2-cycles, i.e., $c_1,c_2,\ldots,c_d\in\mathbb{C}_N$ such that $G^\prime=\prod_{k=1}^dc_kG$.
\end{lemma}
\begin{proof}
	We first consider the case $d=1$.
	By $G^\prime\in{}\mathbb{I}_d(G)$, there exist $A^\prime{}\in{}Aut(G_0)$ and $t\in\mathbb{T}_N$ such that  $G^\prime=gA^\prime{}t{}G_0$.
	Let $c\in\mathbb{S}_N$ be a permutation satisfying $G^\prime=c\cdot{}G$, then we know $G^\prime=c\cdot{}G=c_1gA\cdot{}G_0$ for any $A\in{}Aut(G_0)$.
	Combining the above equations,
	\begin{equation}
	c=gA^\prime{}t(gA^\prime)^{-1}\label{eq:conjugate}
	\end{equation}
	by $A=A^\prime$, implying that $c$ is the conjugate of $t$ by $gA^\prime$.
	Applying Lemma~\ref{thm:conjugate} to~\eqref{eq:conjugate}, we have that $c\in\mathbb{C}_N$ is true.
	Next, we extent it to the case $d=2$.
	In this we have the equation $G^\prime=gA^\prime{}t_1t_2G_0$ holding for $t_1,t_2\in\mathbb{T}_N$.
	We further assume that $c_1,c_2\in\mathbb{S}_N$ and $c_1=(gA^\prime)t_1(gA^\prime)^{-1}$, such that $G^\prime=c_1c_2G$.
	So we can have the following equations:
	\begin{equation}
	G^\prime=gA^\prime{}t_1t_2\cdot{}G_0=c_1gA^\prime{}t_2\cdot{}G_0=c_1gA^\prime{}t_2(gA^\prime{})^{-1}gA^\prime{}G_0=c_1gA^\prime{}t_2(gA^\prime{})^{-1}G,
	\end{equation}
	where the last equation holds for $G=gA^\prime{}G_0$.
	We thus have $c_2=gA^\prime{}t_2(gA^\prime{})^{-1}$, namely, the conjugate of $t_2$ by $gA^\prime{}$.
	Applying Lemma~\ref{thm:conjugate} to these equations, we know that $c_1,c_2\in\mathbb{C}_N$.
	Lemma~\ref{thm:cycle} is proved by extending the above procedure to the cases $d>2$.
\end{proof}
Last, to prove Theorem~\ref{thm:cover}, we see that swapping two vertices using Alg.~\ref{alg:sampling} is equivalent to acting a 2-cycle from $\mathbb{C}_N$ on the vertices of the graph.
Since for all $i_k,j_k\in{}[N],\,i_k\neq{}j_k\,k\in{}[d]$  can be sampled with a positive probability, it deduces that any two-cycles $c_k=(i_k,j_k)\in{}\mathbb{C}_N$ can be drawn with a positive probability using Alg.~\ref{alg:sampling}, covering $\mathbb{I}_d(G)$ according to Lemma~\ref{thm:cycle}.\hfill\qedsymbol

\section{Experiment details and additional results}

\subsection{TNGA+: an Extension of TNGA for TN-PS}
In this subsection, we briefly explain how TNGA+, an extension of TNGA~\cite{li2020evolutionary} for TN-PS, encodes the vertex permutations into chromosomes, which are used to seek for the optimal TN structures by genetic operators.

Figure.~\ref{fig:algorithm} depicts the encoding process.
We encode the structures for TN-PS from two ingredients, the TN-ranks and the permutations, respectively.
For the former, by $\mathbb{F}_{G_0}=\mathbb{Z}^{+,\vert\mathbb{E}_0\vert}$, the ranks can be directly encoded into a string of dimension $\vert\mathbb{E}_0\vert$ with their coordinates in $\mathbb{Z}^{+,\vert\mathbb{E}_0\vert}$. 

For the latter, we randomly embed a permeation into the space $[0,1]^{\vert{}V\vert}$, a set of decimal number vectors, by a \emph{random-key} trick~\cite{bean1994genetic}, which is popularly used to solve the optimal sequencing tasks.
More precisely, the random-key representation encode a permutation with a vector of random numbers from $[0,1]$, and the order of these random numbers reflects the permutation.
For instance, the code $(0.46,\,0.91,0.33)$ would represent the permutation $2\rightarrow{}3\rightarrow{}1$.
Finally,the encoded strings are simply the concatenation of the two ingredients.

\begin{figure*}
	\centering
	\includegraphics[width=140mm, height=40mm]{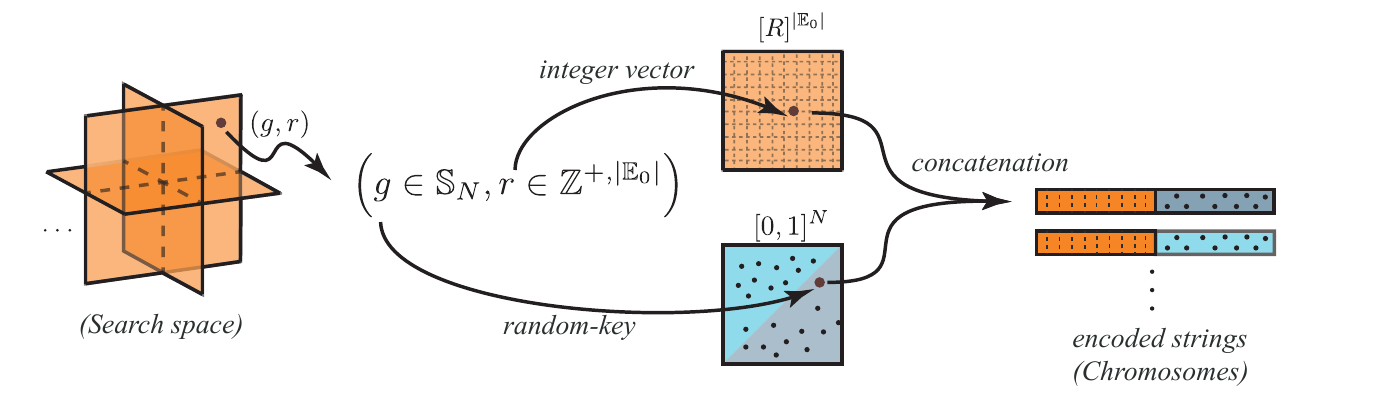}
	\caption{Illustration of how the TN structures for TN-PS are encoded by TNGA+.}
	\label{fig:algorithm}
\end{figure*}
\subsection{Synthetic Data in TT/TR Format}
\noindent\textbf{Configuration of TNGA$^\star$, TNGA+}. 
Both of these two algorithms are based on the GA framework~\cite{li2020evolutionary}. Throughout the TT/TR format synthetic data experiments, they share the same parameters listed as follows. The maximum
number of generations is set to be 30. The population in each generation is set to be 150 under all settings. During each generation in GA, the elimination rate is $36\%$. The reproduction trick in \cite{snyder2006random} is adopted and we set the reproduction number to be 2. Meanwhile, for the selection probability of the recombination operation, we set hyper-parameters $\alpha = 20$ and $\beta = 1$. Moreover, there is a chance of $24\%$ for each gene to mutate after the recombination finished. We initialize the vertices~(core tensors) with each element \emph{i.i.d.} sampled from Gaussian distribution $N(0, 0.1)$. We set the learning rate of the Adam optimizer \citep{kingma2014adam} to 0.001. The decomposition for each individual is repeated 4 times.

\noindent\textbf{Experiment setup of Figure~\ref{fig:convergence} in the manuscript.} In this experiment, the order-8 tensor is selected from the TR structure search experiment. For the order-12 data, we uniformly choose TR-ranks from $\{1,2,3,4\}$ and set the dimension of all tensor modes to 3. The values of vertices are drawn from Gaussian distribution $N(0, 1)$.
After contracting all vertices, we finally uniformly permute the tensor modes in random.
For TNGA+ and TNLS, all the parameters are set the same as in the TR structure search experiment, except that the population of TNGA+ and the sample number of TNLS are set to be 60 or 100.

\begin{figure}
	\centering
	\includegraphics[width=0.9\columnwidth]{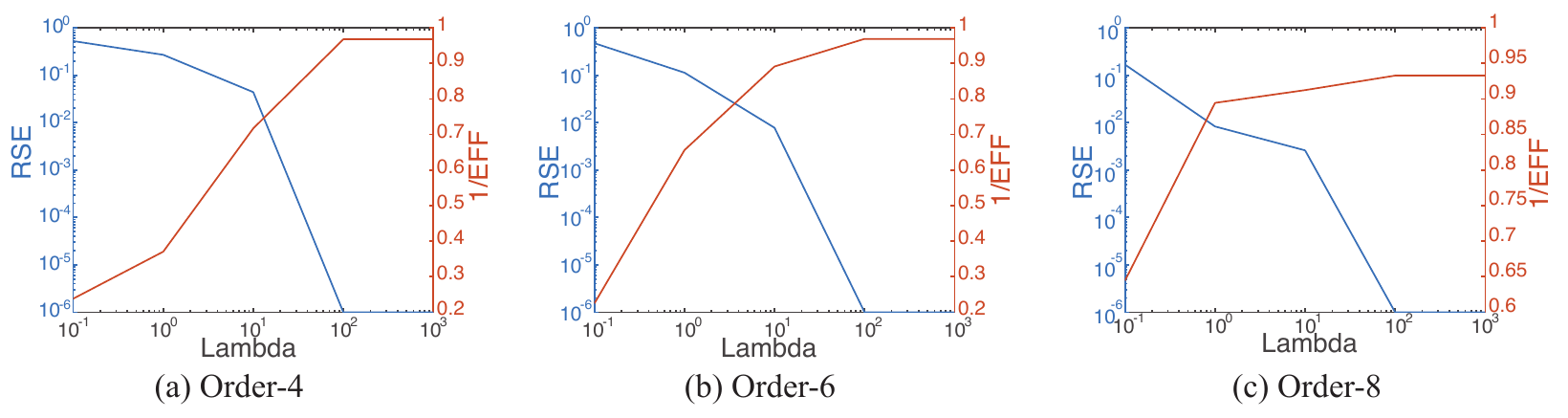}
	\vspace{-0.3cm}
	\caption{RSE and \emph{Eff.} values by TNLS with the varying of the tuning parameter $\lambda$, averaged over the synthetic TT/TR data.
	}\vskip -0.1in
	\label{fig:lambda}
\end{figure}

\noindent\textbf{Trade off between model complexity and approximation accuracy.}
In the experiment, the tuning parameter $\lambda$ given in~\eqref{eq:PTNPS} balances the influence of model complexity and approximation accuracy in the searching process.
Figure~\ref{fig:lambda} shows how RSE and \emph{Eff.} values change with the varying of $\lambda$.
In more details, we choose the values of $\lambda$ from $\{0.1, 1, 10, 100, 1000\}$ and calculate the RSE and \emph{Eff.} averaged over the data used in the synthetic TT/TR data experiments of the order $\{4,6,8\}$, respectively.
Other experiment configuration remains the same as used in the experiment.
We can see from Figure~\ref{fig:lambda}  that the \emph{Eff.} values are larger than $1$ consistently with a wide range of $\lambda$ in all the three orders.
It implies that the TNLS method is relatively stable with the varying of the parameter $\lambda$ \emph{if the tensor is generated with TN models}.
The result is expected since in this case the RSE will decrease dramatically once a good TN structure is found, so the value $\lambda\cdot{}RSE$, the second term of the objective function in~\eqref{eq:PTNPS}, is neglected compared with the first term corresponding to the model complexity.
However, note that the stability would be not held if the tensor is not in low-rank TN formats such as those tensorized natural images.

\subsection{Synthetic Data in Other TN Format}
\noindent\textbf{Data Generation.}
For the synthetic data generation of TTree (order-7) \citep{Ye2019Tensor}, PEPS (order-6) \citep{verstraete2004renormalization}, hierarchical Tucker (HT, order-6) \citep{hackbusch2009new} and multi-scale entanglement renormalization ansatz (MERA, order-8) \citep{cincio2008multiscale,reyes2020multi} which the structures are demonstrated in Figure \ref{fig:TN_strucure}, we first set the dimensions of each tensor mode to 3 and uniformly randomly generate the TN-ranks from $\{1, 2, 3, 4\}$.
Then, each element of the cores is generated \emph{i.i.d.} from Gaussian distribution $N(0,0.1)$. After contracting all vertices, we finally uniformly permute the tensor modes in random. 

\noindent\textbf{Configuration of the comparing methods.}
For TNGA+ and TNLS, the parameters are set as same as the TR structure search experiment, except that for TNGA+ the population in each generation is increased to be 120. Moreover, the coding schemes for HT and MERA are different from TTree and PEPS, which only contain external cores (vertices of color blue). Specifically, for HT and MERA, we fix the permutation of the internal cores~ (vertices of color orange), and therefore only encode the permutation of the external cores. The experimental results including the evaluation numbers of TNGA+ and TNLS are shown in Table  \ref{tab:TN}.
\begin{figure*}[t]
	\centering
	\begin{minipage}[t]{0.12\linewidth}
		\includegraphics[width=1\textwidth]{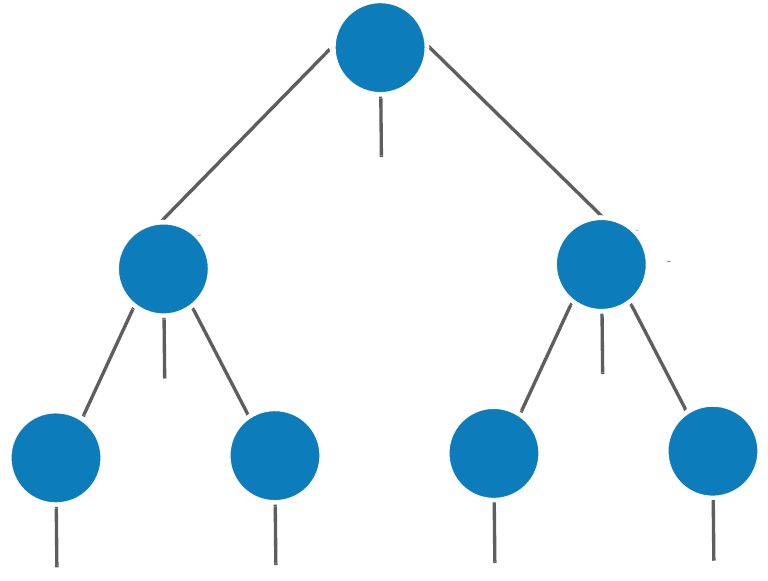}
		\centering{TTree 
		}
	\end{minipage}\hspace{15pt}
	\begin{minipage}[t]{0.12\linewidth}
		\includegraphics[width=1\textwidth]{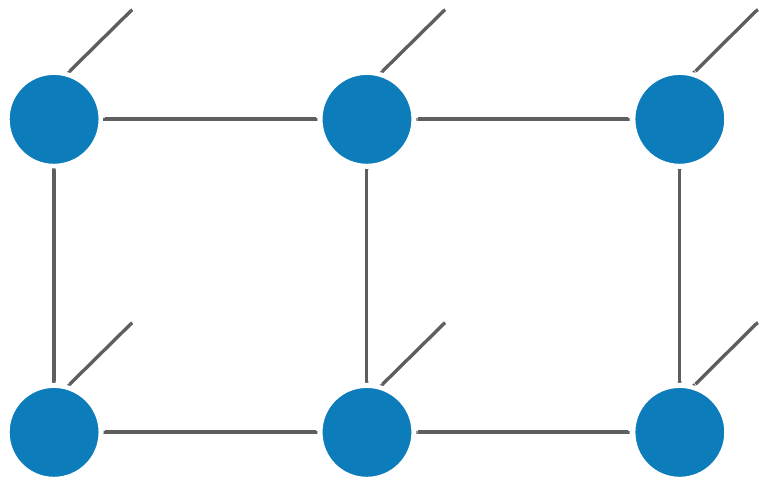}
		\centering{PEPS
		}
	\end{minipage}\hspace{15pt}
	\begin{minipage}[t]{0.18\linewidth}
		\includegraphics[width=1\textwidth]{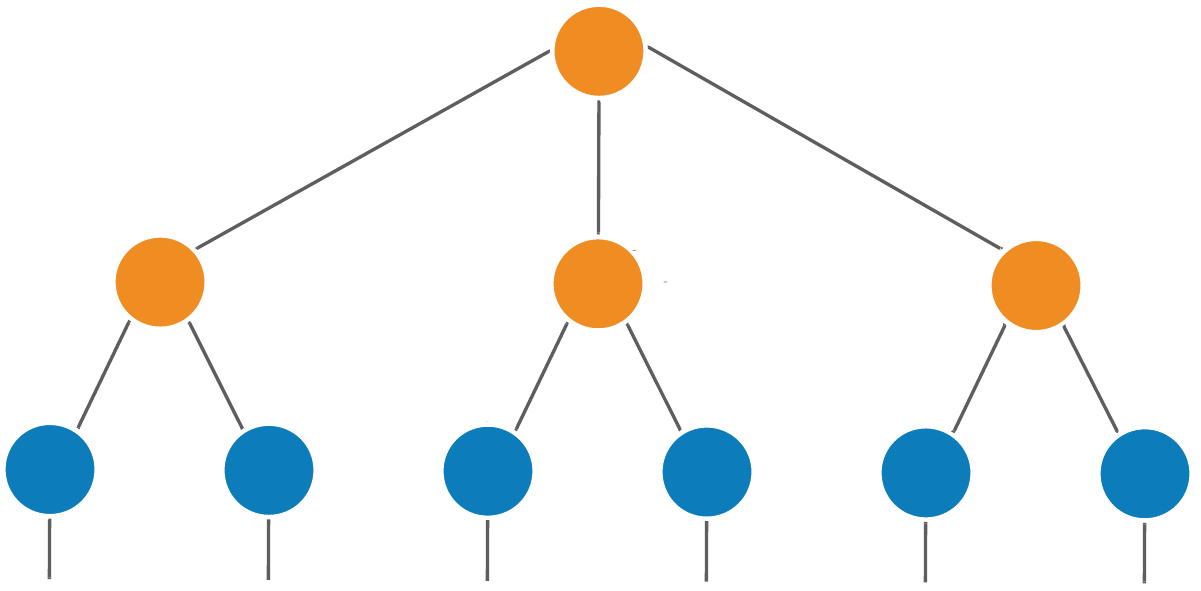}
		\centering{HT
		}
	\end{minipage}\hspace{15pt}
	\begin{minipage}[t]{0.25\linewidth}
		\includegraphics[width=1\textwidth]{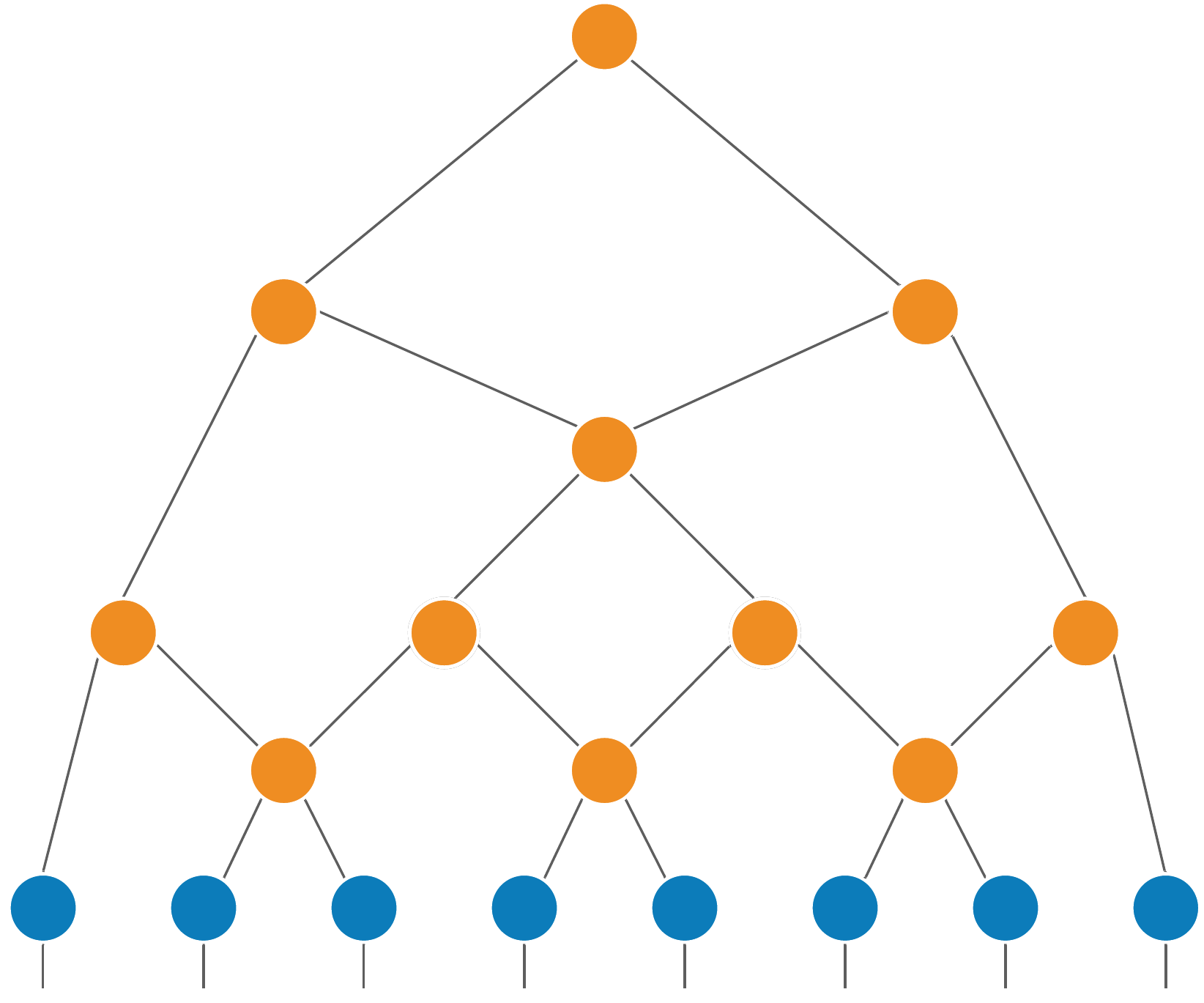}
		\centering{MERA
		}
	\end{minipage}\\
	
	\caption{Illustration of the TN structures applied in the synthetic experiment. The blue nodes with an outer indices indicate the external cores and the orange nodes indicate the internal cores.
	}
	\label{fig:TN_strucure}
\end{figure*}

\subsection{Real-World Data}
\noindent\textbf{Image completion.}
In this experiment, we consider uniformly random missing with the missing
rates $70\%$ and $90\%$. In specific, we firstly use Matlab command ``randperm'' to generate random integer sequence with length that equals to the number of image elements. Then, according to the missing rate, we select a subset of this sequence to generate a binary mask tensor with the same size as the image. Finally, using this mask, we can generate the missing image.
For recovery performance evaluation, we use the RSE of predicted values on the missing entries. 

\begin{figure*}[t]
	\centering
	\begin{minipage}[t]{0.15\linewidth}
		\includegraphics[width=1\textwidth]{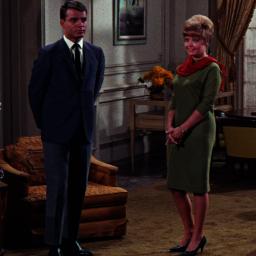}
		\includegraphics[width=1\textwidth]{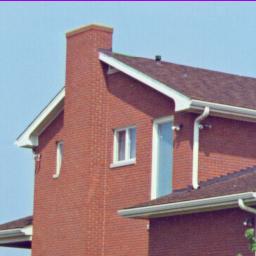}
	\end{minipage}
	\begin{minipage}[t]{0.15\linewidth}
		\includegraphics[width=1\textwidth]{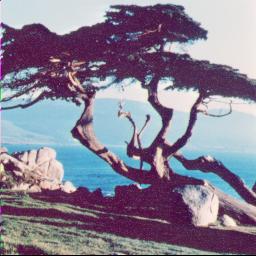}
		\includegraphics[width=1\textwidth]{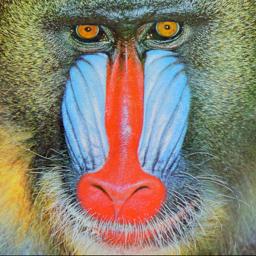}
	\end{minipage}
	\begin{minipage}[t]{0.15\linewidth}
		\includegraphics[width=1\textwidth]{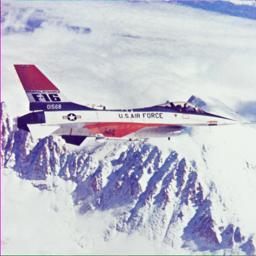}
		\includegraphics[width=1\textwidth]{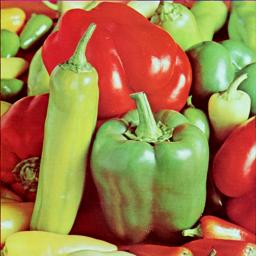}
	\end{minipage}
	\begin{minipage}[t]{0.15\linewidth}
		\includegraphics[width=1\textwidth]{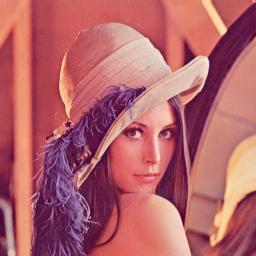}
	\end{minipage}\\
	
	\caption{Illustration of the employed images in image completion experiment. 
	}
	\label{fig:completion}
\end{figure*}

For the proposed TNLS, we set the the maximum iteration $\#Iter$ = 30, and tuning parameters $c_{1}$ = 0.95, $c_{2}$ = 0.9, and the number of sampling $\#Sample$ = 150. Moreover, the rank bound, the learning rate of Adam, and the variance of the Gaussian distribution for core tensors initialization are set to 14, 0.001, 0.1 respectively. For the trade-off parameter $\lambda$, we set it as $0.0008, 0.0007$ for missing rate $0.7, 0.9$.
For TNGA+, the maximum number of the generations is set to be 30. The population in each generation are set to be 300. The elimination rate is  $10\%$ and the reproduction number is set to be 1. Moreover, we set $\alpha = 20$ and $\beta = 1$. The chance for each gene to mutate after the recombination is $24\%$. Other settings, including $\lambda$, Gaussian distribution for initialization, the rank bound and the learning rate, are the same with TNLS.

\noindent\textbf{Image compression.}
In the experiment, we randomly select 10 natural images from the BSD500 \cite{arbelaez2010contour}\footnote{\url{https://www2.eecs.berkeley.edu/Research/Projects/CS/vision/bsds/BSDS300/html/dataset/images.html}}. The images used in this experiment are shown in Figure \ref{fig:live}. We use the Matlab commands ``resize'' and ``rgb2gray'' to turn them into grayscaled images of size $256\times256$, and then rescale them to $\lbrack0,1\rbrack$. Moreover, in this section, we tensorized these images into order-8 tensors by two different ways: a directly reshaping operation denoted by “Reshape” and visual data
tensorization \citep{latorre2005image,bengua2017efficient,yuan2019high}, a image-resolution-based tensorization method, denoted by ``VDT''. 
\begin{figure*}[t]
	\centering
	\begin{minipage}[t]{0.13\linewidth}
		\includegraphics[width=1\textwidth]{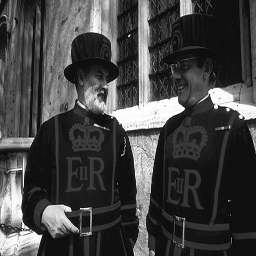}
	\end{minipage}
	\begin{minipage}[t]{0.13\linewidth}
		\includegraphics[width=1\textwidth]{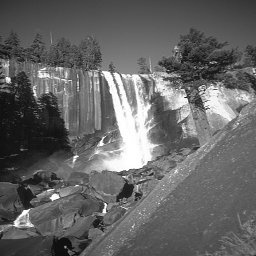}
	\end{minipage}
	\begin{minipage}[t]{0.13\linewidth}
		\includegraphics[width=1\textwidth]{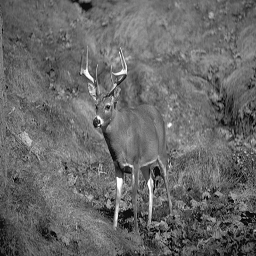}
	\end{minipage}
	\begin{minipage}[t]{0.13\linewidth}
		\includegraphics[width=1\textwidth]{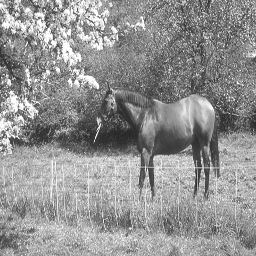}
	\end{minipage}
	\begin{minipage}[t]{0.13\linewidth}
		\includegraphics[width=1\textwidth]{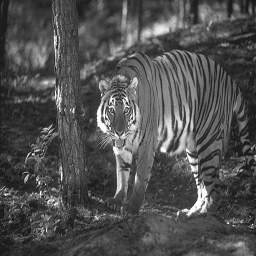}
	\end{minipage}
	\begin{minipage}[t]{0.13\linewidth}
		\includegraphics[width=1\textwidth]{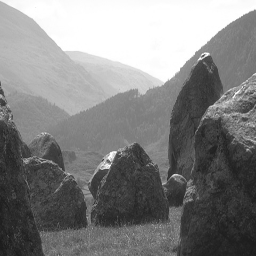}
	\end{minipage}
	\begin{minipage}[t]{0.13\linewidth}
		\includegraphics[width=1\textwidth]{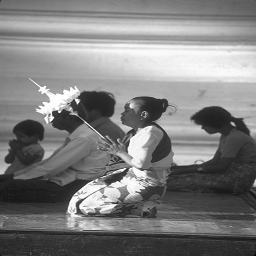}
	\end{minipage}
	\begin{minipage}[t]{0.13\linewidth}
		\includegraphics[width=1\textwidth]{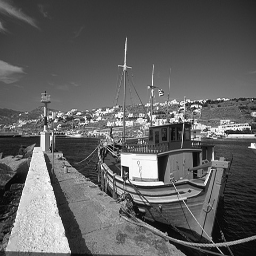}
	\end{minipage}
	\begin{minipage}[t]{0.13\linewidth}
		\includegraphics[width=1\textwidth]{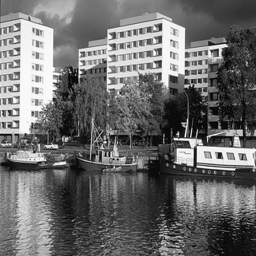}
	\end{minipage}
	\begin{minipage}[t]{0.13\linewidth}
		\includegraphics[width=1\textwidth]{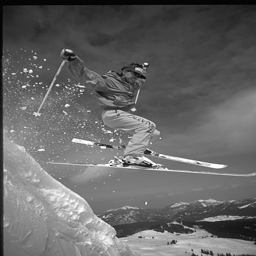}
	\end{minipage}\\
	
	\caption{Illustration of the employed images in image compression experiment. 
	}
	\label{fig:live}
\end{figure*}

For the proposed TNLS, we set the the maximum iteration $\#Iter$ = 20, and tuning parameters $c_{1}$ = 0.95, $c_{2}$ = 0.9, and the number of sampling $\#Sample$ = 150. Moreover, the rank bound, the learning rate of Adam, and the variance of the Gaussian distribution for core tensors initialization are set to be $14, 0.01, 0.1$, respectively. For the trade-off parameter $\lambda$, we set it as 5. For TNGA+, the maximum number of the generations is set to be 30. The population in each generation are set to be 300. The the elimination rate is  $10\%$ and the reproduction number is set to 1. Moreover, we set $\alpha = 25$ and $\beta = 1$. The chance for each gene to mutate after the recombination is $30\%$. Other settings, including $\lambda$, Gaussian distribution for initialization, the rank bound and the learning rate, are the same with TNLS.

\noindent\textbf{Compressing TGP models.}
In this task, we choose three univariate regression datasets from the UCI and LIBSVM archives. The Combined Cycle Power Plant (CCPP)\footnote{\url{https://archive.ics.uci.edu/ml/datasets/Combined+Cycle+Power+Plant}} dataset consists of 9569 data points collected from a power plant with 4 features and a single response. The MG\footnote{\url{https://www.csie.ntu.edu.tw/~cjlin/libsvmtools/datasets/regression.html\#mg}} data have 1385 data points with 6 features and a single response.  The Protein\footnote{\url{https://archive.ics.uci.edu/ml/datasets/Physicochemical+Properties+of+Protein+Tertiary+Structure}} data contain 45730 instances with 9 attributes and a single response. For all the datasets, we randomly choose 80\% of the data for training and the rest for testing, then standardize the training and testing sets respectively by removing the mean and scaling to unit variance, which is the same with settings in TTGP \citep{izmailov2018scalable}.

In this experiment, we aim to demonstrate that our method is capable of searching more efficient structures in this learning task. This is different from the above tasks since we search for TN structures of model parameters, instead of compressing data. Specifically, tensor train Gaussian process (TTGP) tensorizes the variational mean vector in GP to a tensor whose order equals to the number of input features, and the dimension of each order is the number of inducing points on the corresponding feature. In our settings, for CCPP, we choose 12 inducing points on each feature and result in an order-4 tensor of shape $12 \times 12 \times 12 \times 12$. For MG, we choose 8 inducing points and get an order-6 tensor of shape $8 \times 8 \times 8 \times 8 \times 8 \times 8$. For Protein, we choose 4 inducing points and obtain an order-9 tensor of shape $4 \times 4 \times 4 \times 4 \times 4 \times 4 \times 4 \times 4 \times 4$. TTGP uses TT to approximate these tensors. However, the original TTGP are restriced to TT representation and the TT-ranks are treated as pre-defined hyper-parameters. For all the datasets, we set TT-ranks as 10.

To learn more compact representations, we apply the structure searching to TTGP. In particular, we firstly train a TTGP with given TT-ranks and get the TT representation of the variational mean. Then we use our method to search for alternative TR structures of the TT variational mean. Finally, we plug the reparameterized variational mean back into the original TTGP model for inference. In summary, we follow the settings of TTGP except that we reparameterize the TT tensor. 

For the proposed TNLS, we set the the maximum iteration $\#Iter$ = 20, and tuning parameters $c_{1}$ = 0.9, $c_{2}$ = 0.9, and the number of sampling $\#Sample$ = 150, 300, 300 for the TT variational mean of CCPP, MG and Protein regression task. Moreover, the rank bound, the learning rate of Adam, and the variance of the Gaussian distribution for core tensors initialization are set to be $14, 0.001, 0.01$, respectively. For the trade-off parameter $\lambda$, we set it as $\lambda=1\times10^{5},1\times10^{7},1\times10^{3}$ for CCPP, MG and Protein, respectively. For TNGA+, the maximum number of the generations is set to be 30. The population in each generation are set to be 150, 190, 300 for the TT variational mean of CCPP, MG and Protein regression task. The elimination rate is  $30\%$ and the reproduction number is set to 1. Moreover, we set $\alpha = 20$ and $\beta = 1$. The chance for each gene to mutate after the recombination is $30\%$. Other settings, including $\lambda$, Gaussian distribution for initialization, the rank bound and the learning rate, are the same with TNLS.


\end{document}